\documentclass{article}

\usepackage[utf8]{inputenc}
\pdfoutput=1
\usepackage[sc,osf]{mathpazo}
\usepackage[
top    = 2.50cm,
bottom = 2.50cm]{geometry}
\usepackage[sort&compress,numbers]{natbib}
\usepackage[colorlinks=true,citecolor=blue,breaklinks]{hyperref}
\usepackage[hyphenbreaks]{breakurl} 
\usepackage{caption}
\usepackage{subcaption}

\makeatletter
\newlength\aftertitskip     \newlength\beforetitskip
\newlength\interauthorskip  \newlength\aftermaketitskip

\def\maketitle{\par
 \begingroup
   \def\thefootnote{\fnsymbol{footnote}}
   \def\@makefnmark{\hbox to 4pt{$^{\@thefnmark}$\hss}}
   \@maketitle \@thanks
 \endgroup
\setcounter{footnote}{0}
 \let\maketitle\relax \let\@maketitle\relax
 \gdef\@thanks{}\gdef\@author{}\gdef\@title{}\let\thanks\relax}

\def\@startauthor{\noindent \normalsize\bf}
\def\@endauthor{}
\def\@starteditor{\noindent \small {\bf Editor:~}}
\def\@endeditor{\normalsize}
\def\@maketitle{\vbox{\hsize\textwidth
 \linewidth\hsize \vskip \beforetitskip
 {\begin{center} \LARGE\@title \par \end{center}} \vskip \aftertitskip
 {\def\and{\unskip\enspace{\rm and}\enspace}%
  \def\addr{\small\it}%
  \def\email{\hfill\small\tt}%
  \def\name{\normalsize\bf}%
  \def\AND{\@endauthor\rm\hss \vskip \interauthorskip \@startauthor}
  \@startauthor \@author \@endauthor}
}}

\makeatother

\pdfoutput=1                    

\usepackage[utf8]{inputenc} 
\usepackage[T1]{fontenc}    
\usepackage{hyperref}       

\usepackage{url}            
\usepackage{booktabs}       
\usepackage{amsfonts}       
\usepackage{nicefrac}       
\usepackage{microtype}      
\usepackage[utf8]{inputenc}

\usepackage{graphicx}
\usepackage{framed}
\usepackage{amssymb}
\usepackage{amsfonts}
\usepackage{mathrsfs}
\usepackage{mathtools}
\usepackage{array}
\usepackage{amsthm}
\usepackage{verbatim} 
\usepackage{enumerate}
\usepackage{bbm}
\usepackage{mathtools}
\usepackage{commath}
\usepackage{wrapfig}
\usepackage{amsbsy}
\usepackage{subcaption}
\usepackage{float}

\usepackage{url}
\usepackage{amsmath}
\usepackage{amsthm}
\usepackage{algorithm}
\usepackage{mathtools}
\usepackage{amssymb}
\usepackage{tikz}
\usepackage{wrapfig}
\usepackage{hyperref}
\usepackage{enumitem}

\usepackage{algorithm}

\newtheorem{thm}{Theorem}
\newtheorem{definition}[thm]{Definition}

\newtheorem{theorem}[thm]{Theorem}
\newtheorem{lemma}[thm]{Lemma}
\newtheorem{proposition}[thm]{Proposition}
\newtheorem{corollary}[thm]{Corollary}
\usepackage{upgreek}

\newcommand{\grad}{\text{grad}\,}

\newcommand{\R}{{\mathbb R}}

\newcommand{\E}{{\mathbb E}}

\newcommand{\bmu}{\mathbf{\upmu}}

\newcommand{\e}{\varepsilon}

\newcommand{\calH}{{\mathcal H}}
\newcommand{\calP}{{\mathcal P}}

\newcommand{\calW}{{\mathcal W}}

\newcommand{\calF}{{\mathcal F}}

\def\longequals{\mathbin{=\kern-2pt=}}

\newcommand{\beq}{\begin{equation}}
\newcommand{\eeq}{\end{equation}}

\usepackage{algorithmic}
\newtheorem{remark}{Remark}

\providecommand{\nor}[1]{\ensuremath{\left\lVert {#1} \right\rVert}}

\providecommand{\scalT}[2]{\ensuremath{\left\langle{#1},{#2}\right\rangle}}

\title{On the Convergence of Gradient Descent in GANs:   MMD GAN As a Gradient Flow}

\author{\name Youssef Mroueh$^*$ \email{mroueh@us.ibm.com}\\
\addr{IBM Research } \\
\name Truyen Nguyen$^*$  \email{tn8@uakron.edu}\\
\addr{Akron University } \\
\newline
\addr{$*$ Equally contributed } 

}

\begin{document}

\maketitle

\begin{abstract}
We consider the maximum mean discrepancy ($\mathrm{MMD}$) GAN problem and propose a parametric kernelized  gradient flow that mimics the min-max game in gradient regularized  $\mathrm{MMD}$ GAN. We show that this flow provides a descent direction minimizing  the $\mathrm{MMD}$ on a statistical manifold of probability distributions. We then derive an explicit condition 
which ensures  that gradient descent on the parameter space of the generator in gradient regularized $\mathrm{MMD}$ GAN is globally convergent to the target distribution. Under this condition, we give non asymptotic convergence results of gradient descent in MMD GAN. Another contribution of this paper is the introduction of a  dynamic formulation of a regularization of $\mathrm{MMD}$ and demonstrating that the parametric kernelized descent for $\mathrm{MMD}$  is the gradient flow of this functional with respect to the new  Riemannian structure. Our obtained  theoretical result allows ones  to treat  gradient flows for quite general functionals  and thus has potential applications to other types of variational inferences  on a statistical manifold beyond GANs.  Finally, numerical experiments suggest that our parametric kernelized gradient flow stabilizes GAN training and guarantees convergence.  
\end{abstract}
\renewcommand{\thefootnote}{\fnsymbol{footnote}}







\setcounter{equation}{0}
\section{Introduction}
Generative Adversarial Networks (GANs) were introduced in \cite{GANoriginal} and have attracted growing attention in the machine learning community. Implicit Generative models such as GANs can be seen as  learning a distribution  via optimizing a  functional defined on a statistical manifold. The statistical manifold refers to the parametrization of the \emph{Generator}. 
There are a plethora of works on  functionals that are optimized in GANs for example the Jensen-Shanon divergence  in the original work \cite{GANoriginal}; general $\phi$ divergences in \cite{FGAN}; the neural net distance in \cite{Ma_2018}; integral probability metrics such as the Wasserstein order 1 distance  considered in \cite{WGAN}; the maximum mean discrepancy \cite{MMD}  considered in \cite{mmdGAN1,mmdGAN2,li2017mmd,mroueh2017mcgan,Santos_2019_ICCV}. Despite their striking empirical success, the rigorous  understanding of the convergence in \emph{distributional sense} of gradient descent in GANs remains less understood. Much of the theoretical analysis has been dedicated to the stability of the min-max game  via the introduction of gradient regularizers \cite{gulrajani2017improved,SobolevGAN,mescheder2017numerics,roth2017stabilizing,mescheder2018training,Arbel:2018,nagarajan2017gradient}. Min-max convergence rates  for a large class of GANs was studied in \cite{uppal2019nonparametric,liang2018generative}, however these bounds are not specific to  gradient descent in GANs. In this work we aim at understanding the distributional convergence properties of  gradient descent in the context of  $\mathrm{MMD}$ GANs. The work closest to ours is  \cite{bottou2017geometrical} that establishes  global convergence of the generator using the Wasserstein $1$ distance. Nevertheless, gradient descent is not explicitly considered  in \cite{bottou2017geometrical}.\\   

\noindent We summarize our main contributions in this work as follows: 
\begin{itemize}
 \item We introduce in Section \ref{sec:mmdalphabeta} a new gradient regularizer for the MMD. The new regularizer  has the form of  a parametric energy, where the gradient is taken with respect to the generator parameters, instead of the input space as usually considered in previous works.  We call the new proposed discrepancy $\mathrm{MMD}_{\alpha, \beta}$. 

\item  We consider the $\mathrm{MMD}$ GAN problem in Section \ref{sec:mmdflows} and propose a new descent direction in terms of the witness function of the parametric regularized $\mathrm{MMD}_{\alpha, \beta}$. We give in this section detailed descriptions and properties of the corresponding continuous flow.  

\item  We analyze in Section \ref{sec:gradientmmd} the non-asymptotic distributional convergence properties of gradient descent in MMD GAN when using the  $\mathrm{MMD}_{\alpha, \beta}$ witness functions to drive the generator updates.   

\item  We derive in Section \ref{sec:flows} a dynamic formulation of the  $\mathrm{MMD}$ on the statistical manifold of probability distributions and use it to propose a novel regularization   of the $\mathrm{MMD}$, which we call $d_{\alpha,\beta}$. We show that $d_{\alpha,\beta}$ admits a Riemannian metric tensor and  investigate gradient flows for general functionals w.r.t. this structure. Intriguingly, we show that gradient descent in MMD GANs 
driven by the witness function of  $\mathrm{MMD}_{\alpha, \beta}$ coincides with the gradient flow of the $\mathrm{MMD}$ w.r.t. this  new geometric structure.

\item  Finally, we discuss related works in Section~\ref{sec:relatedwork}, and validate experimentally our theoretical findings in Section~\ref{sec:exp}. 

\end{itemize}


\section{Preliminaries}

Let $\Omega$ be an open region in $\R^d$ and let  $\mathcal{H}$ be a reproducing kernel Hilbert space (RKHS) generated by a 
kernel $k(x, y)$ on $\Omega\times \Omega$.
Let $\mathcal{Z}\subset \mathbb{R}^m$  be an open region on a lower dimensional space endowed with a probability distribution $\nu$ on $\mathcal{Z}$.  Let $\Theta$ be a parameter space in  $\mathbb{R}^p$, and 
 $(\theta, z) \to G_{\theta}(z) =\big(G_{\theta}^1(z),..., G_{\theta}^d(z)\big)\in \Omega $ be a generator function defined on $\Theta$. We assume that $k$ is bounded, $G_{\theta}$ is differentiable in $\theta$, and 
\begin{eqnarray}\label{kG-cond-1}
    &\nor{k(x,.) - k(y,.)}_{\mathcal{H}} \leq L\nor{x-y} \quad \mbox{and}\quad \nonumber \\
    &\nor{G_{\theta}(z)- G_{\theta'}(z)} \leq D(z) \nor{\theta-\theta'}
\end{eqnarray}
for some constant $L>0$ and function $D:\mathcal{Z} \to [0,\infty)$ with $\mathbb{E}_{\nu} [D(z)^2] <\infty$. 
In \eqref{kG-cond-1} and throughout the paper, $\langle\cdot, \cdot \rangle$ and $\|\cdot \|$ denote the standard Euclidean inner product and norm, while  $\langle\cdot, \cdot \rangle_\calH$ and $\|\cdot \|_\calH$   denote the inner product and norm on $\calH$. For a probability distribution $\rho$ on $\Omega$,
let $\bmu_{\rho}(x) :=\int k(x,y) \rho(dy)$ denote its  kernel mean embedding. As $k$ is bounded, we have $\bmu_{\rho}\in \mathcal H$
and $\int f(x) \rho(dx) =\langle f, \bmu_\rho \rangle_\calH$ for every $f\in\calH$
(see \cite{muandet2016kernel}). Note that $\bmu_{\rho_1 -\rho_2}=\bmu_{\rho_1} - \bmu_{\rho_2}$ by linearity.

We consider the following statistical manifold of probability distributions:
$$\mathcal{P}_{\Theta}:=\{q_{\theta}= (G_{\theta})_{\#} \nu, \,\, \, \theta \in \Theta \}.$$
The main objective functions considered in this paper are  parametric energy regularizations of
$\mathrm{MMD}(p, q_\theta):= \|\bmu_p - \bmu_{q_\theta}\|_{\calH}$ and the introduction of the following operators plays an important role in understanding these objective functions and associated gradient flows.

\begin{definition}[Matrix Mass]

\noindent Let $J_{\theta}G_{\theta}(z)
 =\big(\frac{\partial G_{\theta}^j(z)}{\partial \theta_i }\big)_{ij}\in \mathbb{R}^{p\times d}$ denote  the Jacobian of $G_{\theta}$ with respect to $\theta$.  
Then for $\theta \in \Theta$, we define a matrix valued kernel $\Gamma_{\theta}$  on $\mathcal{Z}\times \mathcal{Z}$ as follows, $\mbox{for } (z,z') \in \mathcal{Z}\times \mathcal{Z}$:
$$\Gamma_{\theta}(z,z') : =J_{\theta}G_{\theta}(z)^{\top} J_{\theta}G_{\theta}(z') \in \mathbb{R}^{d\times d}\quad .$$
\end{definition}

\begin{definition}[Parametric Grammian - Mass Corrected Grammian of Derivatives]
For  $\theta\in \Theta$, let  
$L_\theta: \mathcal{H} \to  \R^p$ be the operator given by
\[L_\theta(f) :=   \int   J_{\theta}G_{\theta}(z)\nabla f (G_{\theta}(z)) \nu(dz) 
\]
and  $L_\theta^{\top}: \R^p \to \mathcal{H}$ be the operator given by
\begin{align*}
L_\theta^{\top}(v) &:=  \int \langle \nabla_{\theta}[k(G_{\theta}(z)],.), v\rangle  \, \nu(dz) \\
&= \int \langle \nabla_x k(G_{\theta}(z),.), J_\theta G_{\theta}(z)^{\top}v\rangle \, \nu(d z).
\end{align*}
Then the parametric Grammian $D_\theta: \mathcal{H} \to \mathcal{H}$ is defined by
$D_\theta :=L_\theta^{\top}L_\theta$.
\end{definition}
The main properties of these operators are summarized as follows. Proofs are given in the Appendix. 
\begin{proposition}\label{operator} For each $\theta\in \Theta$, we have
\begin{itemize}
    \item[i)] $L_\theta^{\top}$ is the adjoint operator of $L_\theta$, i.e., $\langle L_\theta f, v \rangle =\langle  f, L_\theta^{\top} v \rangle_{\mathcal H}$ for  $f\in \mathcal H$ and $v\in \R^p$.
     \item[ii)] $D_\theta$ is symmetric, i.e., $\langle D_\theta f, g \rangle_{\mathcal H} =\langle  f, D_\theta g \rangle_{\mathcal H}$
     for $f, \, g\in \mathcal H$.
      \item[iii)] $\scalT{f}{D_\theta f}_{\mathcal{H}}=\nor{ \nabla_{\theta} \int f(x)  q_\theta(d x)  }^2=\|\nabla_\theta[\langle f, \bmu_{q_\theta} \rangle_\calH]\|^2\geq 0$. In particular, $D_\theta$ is a positive operator and hence its spectrum is contained in $[0,\infty)$.
      \item[iv)] For $f\in \mathcal H$, we have $(D_\theta f)(x)= \langle D(x,\cdot), f \rangle_{\mathcal H}$ with 
$$ D(y,y') :=  \iint \langle  \partial_{\theta}k(G_{\theta}(z),y) ,  \partial_{\theta}k(G_{\theta}(z'),y') \rangle \nu(d z)\nu(d z')$$, where: $ \partial_{\theta}k(G_{\theta}(z),y) =J_{\theta}G_{\theta}(z) \nabla_{x} k(G_{\theta}(z),y)$. Equivalently $$D(y,y') =\iint  Trace (\nabla_{x} k(G_{\theta}(z),y) \otimes  \Gamma_{\theta}(z,z')  \nabla_x k(G_{\theta}(z'),y') )  \nu(d z)\nu(d z').$$
      
\end{itemize}

\end{proposition}


\section{A Novel Parametric Energy Regularization of MMD}\label{sec:mmdalphabeta}

Let us introduce a  parametric energy regularization of $\mathrm{MMD}$ and this notion of discrepancy  will play a central role in this paper. For  parameters $\alpha,\, \beta\geq 0$, the regularized discrepancy   between a given probability distribution $p$ on $\Omega$  and a parametric distribution $q_{\theta} \in \mathcal{P}_{\Theta}$ is defined  by
\begin{align}\label{MMD}
&\mathrm{MMD}_{\alpha,\beta}(p,q_{\theta}) :=\sup_{f\in E_{\alpha,\beta}} \left\{\int_{\Omega} f(x)\, p(dx)- \int_{\Omega} f(x)\,q_\theta(dx)\right\}
\end{align}
with  $E_{\alpha,\beta}:=  \Big\{ f\in\mathcal{H}: \,  \,  \alpha \nor{ \nabla_{\theta} \int f(x)  q_\theta(d x)  }^2_{\mathbb{R}^p}+\beta \nor{f}^2_{\mathcal{H}}\leq \frac12\Big\}$.  The case $\alpha=0$ and $\beta=1/2$  corresponds to the $\mathrm{MMD}$ \cite{MMD}, while the case $\alpha=1/2$ and $\beta=0$ shares  some similarity  with  the usual kernelized  Sobolev discrepancy \cite{SD}. The main difference in this definition  with the  Sobolev discrepancy is that the parametric energy $\|\nabla_\theta[\langle f, \bmu_{q_\theta} \rangle_\calH]\|_{\R^p}^2= \scalT{f}{D_{\theta}f}_{\cal{H}}$ is used in place of the standard energy $\|\nabla f\|_{L^2_{q_\theta}}^2$.

\begin{remark} Note that our regularization for  $\mathrm{MMD}$ while it shares similarities with WGAN-GP \cite{gulrajani2017improved}, it is different since our gradient penalty is with respect to the generator parameter whereas it is with respect to the input in WGAN-GP.
\end{remark}

 Hereafter, $I: \mathcal H \to \mathcal H$ denotes the identity operator. Then it follow from property iii) in Proposition~\ref{operator} that $\alpha D_\theta + \beta I$ is invertible whenever $\alpha\geq 0$ and $\beta>0$.
Proposition~\ref{operator} also allows us to express the constraint $ E_{\alpha,\beta}$ in \eqref{MMD} as
\begin{equation}\label{E}
    E_{\alpha,\beta}=  \Big\{ f\in\mathcal{H}: \,  \,  2\scalT{f}{(\alpha D_\theta + \beta I) f}_{\mathcal{H}} \leq 1\Big\}.
\end{equation}
This constraint  can be interpreted as a regularization through the following unconstrained formulation.
\begin{proposition}\label{duality} Let $\Delta_{\theta}(f) := \int f(x)\, p(dx)- \int f(x)\,q_\theta(dx)$.
For $\alpha\geq 0$ and $\beta>0$, we have 
\begin{align*}
&\mathrm{MMD}_{\alpha,\beta}(p,q_{\theta})^2\\
&= \sup_{f\in \mathcal{H}} \left\{ \Delta_{\theta}(f)  - \frac{\alpha}{2} \|\nabla_\theta[\langle f, \bmu_{q_\theta} \rangle_\calH]\|_{\R^p}^2-\frac{\beta}{2} \nor{f}^2_{\mathcal{H}}\right\}\\
&=\frac12  \scalT{\bmu_{p -q_\theta}}{(\alpha D_\theta  +\beta I)^{-1} \bmu_{p -q_\theta}}_{\mathcal{H}}.
\end{align*}
Moreover, the witness function $f^*$ realizing the above supremum is given by:
\begin{equation}\label{witness-fn}
   (\alpha D_\theta + \beta I) f^* = \bmu_{p -q_\theta}. 
\end{equation}
\end{proposition}

The next result shows that the regularized $\mathrm{MMD}_{\alpha,\beta}$ is  upper bounded by the $\mathrm{MMD}$, and gives a characterization on distributions for which the two discrepancies are the same. 
\begin{corollary}\label{upper-bound}
For $\alpha\geq 0$  and $\beta>0$, we have 
\[\sqrt{2\beta} \mathrm{MMD}_{\alpha,\beta}(p,q_{\theta})\leq 
 \mathrm{MMD}(p,q_{\theta}).
\]
In case  $\alpha>0$,  the equality happens if and only if $D_\theta \bmu_{p-q_\theta} =0$.
\end{corollary}

\section{Generative Adversarial Networks via Parametric Regularized Flows}\label{sec:mmdflows}
 Let $p$ be the target distribution which is a probability measure  on $\Omega$.  Consider the  functional $\mathcal{F}(q_{\theta}) = \frac12 \text{MMD}^2(p,q_{\theta}) = \frac{1}{2}\nor{\bmu_{p -q_\theta}}^2_{\mathcal{H}}.$
We now focus on  the MMD GAN problem:
\begin{align*}
\min_{q\in \mathcal{P}_{\Theta}} \mathcal{F}(q)&=\min_{\theta \in \Theta } \mathcal{F}(q_{\theta})\\
&=\min_{\theta \in \Theta }  \frac{1}{2} \nor{\mathbb{E}_p k(x,.)- \mathbb{E}_{z\sim \nu} k(G_{\theta}(z),.) }^2_{\mathcal{H}}.
\end{align*}
This problem  has been investigated  in several works \cite{mmdGAN1,mmdGAN2,li2017mmd,SD,arbel2019maximum}.
In this section we propose a new descent direction in the parameter space of the generator.\\

\noindent \textbf{Continuous Descent.}        
\noindent We consider the following dynamic for any sequence of functions $f_{t}\in \mathcal{H}, t\geq 0$: 
\begin{equation}\label{theta-eq1}
    \frac{d \theta_{t}}{dt} = L_{\theta_t}(f_t)=\int  J_{\theta_t }G_{\theta_t}(z')  \nabla_{x}f_t(G_{\theta_t}(z'))\nu(d z').
\end{equation}

For a given $z\in \mathcal{Z}$, the dynamic of the generator is as follows
\begin{align}
\frac{d G_{\theta_{t}} (z)  }{d t }&=  J_{\theta_t}G_{\theta_{t}}(z)^{\top}  \frac{d \theta_{t}}{dt}\\
 &=\int J_{\theta_t}G_{\theta_{t}}(z)^{\top}JG_{\theta_t}(z')  \nabla_{x}f_t(G_{\theta_t}(z'))\nu(d z')  \nonumber \\
&= \int \Gamma_{\theta_{t}}(z,z') \nabla_{x}f_t(G_{\theta_t}(z'))\nu(dz'). 
\label{eq:Generatordynamic}
\end{align}
While the dynamic of particles is usually given by  velocities defined  at each particle, the generator's dynamic at a given $z$ is the average of the mass corrected velocities of all other samples from the generator. The mass correction is driven by the matrix valued kernel $\Gamma_{\theta_{t}}$ that defines a similarity in the hidden space $\mathcal{Z}$. 
For example in particles descent such as Sobolev Descent \cite{SD} the dynamic of particles $X_t$ is given by:
$$\frac{dX_t}{dt}=\nabla_x f(X_t),$$
 and  this simple advection dynamic is to be contrasted with the generator dynamic \eqref{eq:Generatordynamic}.\\

Using  the dynamic of the generator in \eqref{eq:Generatordynamic} and item iv) in Proposition~\ref{operator}, we also have the following dynamic of the mean embedding:
\begin{align}\label{rate-kernel-embedding}
 &\frac{d}{dt} \bmu_{q_{\theta_t}}=\frac{d}{dt} \int k( G_{\theta_t}(z), . ) \nu(dz) \nonumber\\
&= \int  \scalT{\nabla_x k( G_{\theta_t}(z), . )}{ \frac{d G_{\theta_{t}} (z)  }{d t } }  \nu(dz)\nonumber\\ 
 & =\iint  \scalT{\nabla_x k( G_{\theta_t}(z), . )}{  \Gamma_{\theta_{t}}(z,z') \nabla_{x}f_t(G_{\theta_t}(z'))}\nu(dz')  \nu(dz)\nonumber\\
&=D_{\theta_t} f_{t}.
\end{align}

Given the dynamic of the mean embedding in \eqref{rate-kernel-embedding} it is easy to derive the dynamic of the $\mathrm{MMD}$ distance:
\begin{eqnarray}
\frac{ d \mathcal{F}(q_{\theta_{t}})}{dt}
= \scalT{\bmu_{p -q_{\theta_t}}}{ - \frac{d}{dt} \bmu_{q_{\theta_t}} }_{\mathcal{H}}
= - \scalT{\bmu_{p -q_{\theta_t}}}{D_{\theta_t} f_{t} }_{\mathcal{H}}.
\label{eq:mmddynamic}
\end{eqnarray}

Let us consider the following choices for the sequence $f_{t}$:\\

\noindent $\bullet$ \textbf{Witness functions of $\mathrm{MMD}$.} In that case, we set  $f_{t}= \bmu_{p-q_{\theta_t}}$. Using the generator updates given in \eqref{theta-eq1}, we have therefore:
$$\frac{ d \mathcal{F}(q_{\theta_{t}})}{dt}= - \scalT{\bmu_{p -q_{\theta_t}}}{D_{\theta_t} \bmu_{p -q_{\theta_t}} }_{\mathcal{H}} \leq 0.$$
This is a valid descent direction, and is similar in spirit to the MMD flows of \cite{Arbel2020Kernelized}. Nevertheless as shown for the particles case in \cite{Arbel2020Kernelized}, it does not lead to convergence. In the discrete case, \cite{Arbel2020Kernelized} introduced a noising scheme that has convergence guarantees.\\ 

\noindent $\bullet$ \textbf{Witness functions of $\mathrm{MMD}_{1,0}$.} In that case,  let us assume that  solutions $f_t$   of
$D_{\theta_t}f_{t}= \bmu_{p-q_{\theta_t}}$ exist. Then by  using $d \theta_{t}= L_{\theta_t}(f_t) dt$
we obtain: 
\[\frac{ d \mathcal{F}(q_{\theta_{t}})}{dt}= - \scalT{\bmu_{p-q_{\theta_t}}}{\bmu_{p-q_{\theta_t}}}=  - 2 \mathcal{F}(q_{\theta_{t}}).\]
While this seems to be  the ideal choice as it gives us an exponential convergence, it comes with the caveat that $D_{\theta_{t}}$ may be singular and hence we have either no solution or infinitely many solutions for $f_t$. These derivations and the singularity issue of operator $D_\theta$ motivated the introduction of  $\mathrm{MMD}_{\alpha,\beta}$, and we consider hereafter  its flows.\\

\noindent $\bullet$ \textbf{Witness functions of $\mathrm{MMD}_{\alpha,\beta}$.} Let 
$f_{t}$ be the unique  witness function   of $\mathrm{MMD}_{\alpha,\beta}$ between $p$  and  $q_{\theta_{t}}= (G_{\theta_t})_{\#}\nu$ given by
\begin{equation}\label{theta-eq2}
(\alpha D_{\theta_{t}}+ \beta I ) f_{t} = \bmu_{p -q_{\theta_t}}.
\end{equation}
Theorem~\ref{theo:statisticalManifoldDescent} below gives the dynamic of the $\mathrm{MMD}$ when the generator parameters are updated according to Equation \eqref{theta-eq1} with $f_t$ being the  $\mathrm{MMD}_{\alpha,\beta}$ witness functions  given in Equation~\eqref{theta-eq2}.

\begin{theorem} [Parametric Regularized Flows Decrease the $\mathrm{MMD}$ Distance] Assume that $\alpha,\, \beta>0$.
Then the dynamic \eqref{theta-eq1}--\eqref{theta-eq2} defined by the witness function of the parametric regularized  $\mathrm{MMD}$ decreases the functional $\calF(q_\theta)$:
\begin{equation}\label{rate}
\frac{ d \mathcal{F}(q_{\theta_{t}})}{dt} = -\frac{2}{\alpha} \Big[\mathcal{F}(q_{\theta_{t}}) - \beta \,  \mathrm{MMD}_{\alpha,\beta}(p,q_{\theta_t})^2 \Big] \leq 0.
\end{equation}
Moreover, we have $\frac{ d \mathcal{F}(q_{\theta_{t}})}{dt} <0$  if and only if  $D_{\theta_t} \bmu_{p-q_{\theta_t}}\neq 0$.
\label{theo:statisticalManifoldDescent}
\end{theorem}

We see from Theorem~\ref{theo:statisticalManifoldDescent} that  $\mathrm{MMD}_{\alpha,\beta}$  witness functions alleviate the singularity issue of $D_{\theta_t}$, but slows down the convergence by introducing a damping term that is proportional to $\mathrm{MMD}_{\alpha,\beta}^2$.

\section{Non-Asymptotic Convergence Of Gradient Descent In MMD GAN }\label{sec:gradientmmd}

In Section \ref{sec:mmdflows}, we  showed that $\mathrm{MMD}_{\alpha,\beta}$ witness functions provide descent directions for continuous MMD GAN. In this section we turn to (discrete) gradient descent in the parameter space of the generator, and give non asymptotic convergence results for gradient descent  in regularized MMD GANs.\\
 
\noindent \textbf{Discrete  Descent Directions.} We would like to identify  directions of $\theta$ along which the functional $\calF(q_\theta) $ decreases its value. For this, let us  compute the rate 
$\frac{d}{d\e}\big|_{\e=0}   
\calF(q_{\theta +\e v})$ for each vector $v\in \R^p$.  
\begin{lemma}\label{compute-rate}
Let $\theta\in\Theta$. Then for any vector $v\in \R^p$, we have 
\begin{equation*}
 \frac{d}{d\e}  
\calF(q_{\theta +\e v})
 =- \langle  \bmu_{p-q_{\theta+\e v}}, L_{\theta+\e v}^{\top} v \rangle_{\calH}\quad \mbox{for any}\quad \e\geq 0.
\end{equation*}
     \end{lemma}
  
 Lemma~\ref{compute-rate}  implies that $v$ is a descent direction of $\calF(q_\theta )$ if and only if $\langle  \bmu_{p-q_{\theta}}, L_\theta^{\top} v \rangle_{\calH}>0$. The  next result gives one such direction.
\begin{proposition} \label{rate0}
Let $\theta \in \Theta$, and assume that $D_\theta \bmu_{p-q_{\theta}} \neq 0$. Let $v^*= L_\theta f$ with $f\in \calH$ being the solution of $(\alpha D_\theta +\beta I)f =\bmu_{p-q_{\theta}}$. Then $v^*$ is a descent direction of $\calF(q_\theta )$. Precisely, we have
\begin{equation*}
 \frac{d}{d\e}\Big|_{\e=0}   
\calF(q_{\theta +\e v^*})
 =- \Big[\alpha \| D_\theta f\|_{\calH}^2  
+\beta \langle f, D_\theta f \rangle_{\calH}\Big] <0.
\end{equation*}
In particular, 
$\calF(q_{\theta +\e v^*}) < \calF(q_{\theta})$
if $\e>0$ is small enough.
 \end{proposition}

\noindent \textbf{Discrete Time Descent for MMD GAN.}
Hereafter, we denote $\|h\|_\calH : = \Big(\sum_{i=1}^d \|h_i\|_\calH^2\Big)^\frac12 $ for $h=(h_1,...,h_d)\in \calH^d$ and $\|A\| := (\sum_{i, j} a_{ij}^2)^\frac12 $ for a matrix $A=(a_{i j})$. The next result holds under the following extra conditions for $k$ and $G$:
\begin{align}
&\| \nabla_x k(z_1,.) - \nabla_x k(z_2,.)\|_\calH\leq \tilde L \, \|z_1-z_2\|, \label{extra-1} \\
&\|J_{\theta} G_{\theta}(z) - J_{\theta'} G_{\theta'}(z)\| \leq \tilde D(z)\,  \|\theta -\theta'\|\label{extra-2}
\end{align}
with $\E_{\nu}[\tilde D(z)]<\infty$. 

In Theorem \ref{discrete-conv} below we  find conditions  under which we can achieve  global convergence in $\mathrm{MMD}$ of gradient descent in MMD GAN. When the kernel is  characteristic, this is equivalent to the global weak convergence  of gradient descent in MMD GAN to the target distribution.  
Gradient descent for MMD GAN is given by the following updates: for all  $\ell\geq 1$,the witness functions between $p$ and $q_{\theta_{\ell}}= (G_{\theta_{\ell}})_{\#}\nu$ update:
\begin{equation}
f_{\ell} := (\alpha_{\ell} D_{\theta_{\ell}}+\beta_{\ell} I )^{-1}\bmu_{p -q_{\theta_{\ell}}}
\label{witness}
\end{equation}
and the generator update,
\begin{equation}
\theta_{\ell+1} := \theta_{\ell}+ \varepsilon_{\ell} L_{\theta_{\ell}} (f_{\ell}),
\label{eq:update_theta}
\end{equation}
where $\alpha_{\ell},\beta_{\ell}$ is a sequence of regularization parameters, and $\varepsilon_{\ell}$ is  a sequence of learning rates.
 
\begin{theorem}\label{discrete-conv}
Assume that $k$ and $G$ satisfy \eqref{kG-cond-1}  and 
 \eqref{extra-1}--\eqref{extra-2}. Let $\lambda_{i}(\theta)>0$ be the smallest non-zero eigenvalue of $D_{\theta}$, and define $a(\theta,f) :=  1 - \frac{ || P_{\mathrm{Null}(D_{\theta})}f ||^2_{\calH}}{ ||f||^2_{\calH} }$ with  $P_{\mathrm{Null}(D_{\theta})}$ being the projection to the null space of $D_{\theta}$. Let $\calF(q_\theta)=\frac12\mathrm{MMD}(p,q_\theta)^2$.
Consider the gradient descent updates given in \eqref{witness} and \eqref{eq:update_theta}. 
Let  $0< \tau< 1$ and $\theta_1\in \Theta$ be the starting point chosen such that:
\begin{equation}
a(\theta_{1},f_{1}) > \tau.
\end{equation}
The sequence $\varepsilon_{\ell}$ is chosen so that the following two  conditions are satisfied for each $\ell$:
\begin{equation}\label{null-condition}
a(\theta_{\ell+1},f_{\ell+1}) > \tau
\end{equation}
and
\begin{equation}\label{step-size-cond}
2 C (2\beta_{\ell})^{-1} \big(1 + \sqrt{\calF(q_{\theta_1})}\big) \leq \e_\ell^{-1}
\end{equation}
with $C>0$ depending only on the constants $C_1, \, C_2, \, C_3, \, C_4$ given in Lemma~\ref{Lipschitz-est}.
 
Under the condition \eqref{null-condition} on $\e_{\ell}$ we have $\tau< a(\theta_{\ell},f_{\ell})\leq1$. Let $\chi_j := \frac{ \lambda_i(\theta_j)a(\theta_{j},f_{j})}{\alpha_j \lambda_i(\theta_j)a(\theta_{j},f_{j}) +\beta_j} >0$. Then we have:
\begin{align*}
    \calF(q_{\theta_{\ell +1}})   
    &\leq  \calF(q_{\theta_1 }) \exp(- \sum_{j=1}^{\ell}\e_j \chi_j) \quad   \forall \ell\geq 1.
  \end{align*}
In particular for $\alpha_{\ell}\leq \frac{ \tau}{2}$ and $\beta_{\ell}=\alpha_{\ell}\lambda_{i}(\theta_{\ell})$, we obtain $\chi_{\ell}\geq 1$ and it follows that:
\begin{align*}
    \calF(q_{\theta_{\ell +1}})   
    &\leq  \calF(q_{\theta_1 }) \exp(- \sum_{j=1}^{\ell}\e_j ).
  \end{align*}
\end{theorem}

  Consequently if 
  $\sum_{j=1}^\infty \e_j  =+\infty$ (meaning that $\e_j$ decays as $1/\sqrt{j}$ for e.g)
 and conditions \eqref{null-condition} and \eqref{step-size-cond} hold, then
we obtain $\mathcal{F}(\theta_{\ell})\to 0 $ as $\ell \to \infty$.
As we see from Theorem \ref{discrete-conv}, not all learning rates are admissible. At a given iteration $\ell$, we select a learning rate $\varepsilon_{\ell}$, so that 
$f_{\ell+1} \notin \mathrm{Null}(D_{\theta_{\ell+1}})$ (ensured by condition \eqref{null-condition}) and so that  \eqref{step-size-cond}  holds as well. To understand condition  \eqref{null-condition}, recall from Theorem \ref{theo:statisticalManifoldDescent} that having  $f_{\ell}$ not in the null space of $D_{\theta_{\ell}}$  means that we have a strict descent.
(It is easy to see that  $f_{\ell} \notin \mathrm{Null}(D_{\theta_{\ell}})$ if and only if $\bmu_{p-q_{\theta_{\ell}}} \notin \mathrm{Null}(D_{\theta_{\ell}}$)).

\section{Parametric Kernelized Flows for a General Functional }\label{sec:flows}
The flow of  the MMD functional (i.e. $\frac12 \mathrm{MMD}(p, q_\theta)^2$) analyzed in the previous sections is driven by the gradient of the witness function between $p$ and $q_\theta$ of the discrepancy $\mathrm{MMD}_{\alpha,\beta}$. In this section we discover a Riemmanian structure on the statistical manifold of probability distributions and show that the continuous gradient descent in $\mathrm{MMD}$ GANs described in Section~\ref{sec:mmdflows} coincides with the gradient flow of the functional with respect to this new geometric structure. We also develop a rigorous theory for treating gradient flows of  general functionals and thus open a way for other types of variational inferences beyond GANs. 

\subsection{Dynamic Formulation For \texorpdfstring{$\mathrm{MMD}$}{Lg} on a Statistical Manifold}
The following result gives a dynamic formulation of $\mathrm{MMD}$ and allows us to discover a 
Riemmanian structure associated to $\mathrm{MMD}$. This is analogous to Benamou--Brenier dynamic formulation of the  Wasserstein of order 2 \cite{dynamicTransport}:
\begin{eqnarray}
& {W}^2_2(p,q)= \inf_{q_t,f_t} \int_{0}^1 \int \nor{\nabla_x f_{t}(x)}^2 q_{t}(dx) dt\nonumber\\
& \text{ s.t } \frac{\partial q_t(x)}{\partial t}=-div(q_t \nabla_x f_t(x))~~ q_0=q, \, q_1=p.
 \label{eq:Benamou}
 \end{eqnarray}
 The main difference is that their flows are with respect to the standard $L^2_{q_\theta}$ energy, while ours as explained in Section \ref{sec:regMMD}, are driven by the parametric energy $||\nabla_\theta [\langle f, \bmu_{q_\theta}\rangle_\calH]\|^2$.

\begin{theorem} [Dynamic $\mathrm{MMD}$ on a Statistical Manifold]\label{dynamic}
Assume that for any $\theta_0, \theta_1\in \Theta$, there exists a path $(\theta_t, f_t)_{t\in [0,1]}$ such that $\theta_{t=0}=\theta_{0}$, $\theta_{t=1}=\theta_1$,  $f_{t}\in \mathcal{H}$, and \begin{equation}
\partial_t \theta_{t} = L_{\theta_t} f_t \mbox{ and } D_{\theta_t} f_t= 2 \bmu_{q_{\theta_1}- q_{\theta_{t}}}\,\,\,\,\, \forall t\in [0,1].
\label{eq:assump}
\end{equation}
Let $q_{\theta_0}$ and $q_{\theta_1}$ be two probability measures in $\mathcal{P}_{\Theta}$. Then we have the following dynamic form of $\mathrm{MMD}$ between distributions defined on the statistical manifold $\mathcal{P}_{\Theta}$:
$$ \mathrm{MMD}^2( q_{\theta_0},q_{\theta_1})=\min_{(\theta_t, f_{t})} \left\{\int_{0}^1 \nor{D_{\theta_t}f_t}^2_{\mathcal{H}} dt \right\},$$
$$\, \partial_t \theta_{t} = L_{\theta_t}f_t, f_{t}\in \mathcal{H},\theta_{t=0}=\theta_{0}, \theta_{t=1}=\theta_1  .$$
\end{theorem}

Theorem~\ref{dynamic} is proven under Assumption \ref{eq:assump} that guarantees the existence of a solution. This assumption is not realistic since $D_{\theta_{t}}$ can be singular. Nevertheless, we state this theorem to motivate the introduction in the next section of the dynamic form  of a  regularized $\mathrm{MMD}$ that alleviates this singularity issue.

\subsection{Regularized \texorpdfstring{$\mathrm{MMD}$}{Lg} and Gradient Flows on a Statistical Manifold}\label{sec:regMMD}
Motivated by the result in Theorem~\ref{dynamic} and to alleviate the singularity issue in Assumption  \eqref{eq:assump},
we define  the following regularized version of $\mathrm{MMD}$: 
\begin{definition}[Regularized $\mathrm{MMD}$ on a Statistical Manifold] 
Let $\alpha, \, \beta>0$. Define
\begin{align*}
   d_{\alpha,\beta}( q_{\theta_0},q_{\theta_1})^2 &=\min_{\theta_t, f_{t}} \int_{0}^1 \Big(\alpha \nor{D_{\theta_t} f_t}^2_{\mathcal{H}} +\beta \scalT{f_{t}}{D_{\theta_t}f_t}_\calH \Big) dt, \\
&\partial_t \theta_{t} = L_{\theta_t}f_t,\,  f_{t}\in \mathcal{H},\, \theta_{t=0}=\theta_{0},\, \theta_{t=1}=\theta_1.  
\end{align*}
\end{definition}

 Note that the regularization we introduced here is the parametric energy $$||\nabla_\theta [\langle f_t, \bmu_{q_\theta}\rangle_\calH]|_{\theta=\theta_{t}}\|^2= \scalT{f_{t}}{D_{\theta_t}f_t}_\calH,$$ which plays a similar role as  the kinetic energy in Benamou-Brenier's formula. The evolution of $\theta_{t}$ in our  form is analogous to the continuity equation in the $\mathcal{W}_2$ dynamic form \cite{dynamicTransport}. Conditions on the kernel and the generator family  under which  we can guarantee existence of the solution  for this problem are out of the scope of this work, since our interest in this heuristic Riemannian structure is solely in order to define an appropriate tangent space and Riemannian metric tensor. We leave the analysis of $d_{\alpha,\beta}$ to a future work.

 This dynamic formulation gives rise to the following  Riemannian metric tensor
 on the tangent space of $\Theta$: for $\theta\in \Theta$, let 
\begin{align*}
g_{\theta}(\xi_1,\xi_2) &:= \alpha \scalT{D_\theta \varphi_1}{D_\theta \varphi_2}_{\mathcal{H}}+ \beta \scalT{\varphi_1}{D_\theta \varphi_2}_{\mathcal{H}} \\
&= \scalT{(\alpha D_\theta+\beta I)\varphi_1}{D_\theta \varphi_2}_{\mathcal{H}} 
\end{align*}
where $\xi_i=  L_\theta(\varphi_i)=\nabla_\theta[\langle \varphi_i, \bmu_{q_\theta} \rangle_\calH]\in \R^p$ with $\varphi_i \in \mathcal H$ ($i=1,2$). We note that $g_{\theta}(\xi_1,\xi_2)=g_{\theta}(\xi_2,\xi_1)$ due to the symmetry of $D_\theta$ (see property ii) of Proposition~\ref{operator}).
Then it follows that
$$ d_{\alpha,\beta}( q_{\theta_0},q_{\theta_1})^2=  \min_{\theta_t, f_t} \Big\{\int_{0}^1 g_{\theta}(\partial_t \theta_t, \partial_t \theta_t)dt
\Big\},$$
$$ \partial_t \theta_{t} = L_{\theta_t}f_t,\, f_{t}\in \mathcal{H},\, \theta_{t=0}=\theta_{0},\, \theta_{t=1}=\theta_1. $$
Let us assume that $\alpha,\, \beta>0$ from now on. For a functional $\calF: \calP_\Theta \to \R$, let $\underset{ d_{\alpha,\beta}}{\grad} \mathcal{F}(q_\theta)$ denote the gradient of 
$\calF$ with respect to the metric  $d_{\alpha,\beta}$. That is, $\underset{ d_{\alpha,\beta}}{\grad} \mathcal{F}(q_\theta)$ is a vector in $\R^p$ satisfying
\begin{equation}\label{grad-def}
\frac{d \calF(q_{\theta_t})}{dt}
\Big|_{t=0} =g_{\theta} (\underset{ d_{\alpha,\beta}}{\grad} \mathcal{F}(q_{\theta}),\xi)
\end{equation}
for every  differentiable curve   $t\mapsto \theta_t\in \Theta$  with $\theta_{t=0} =\theta$ and 
$
 \partial_t \theta_t|_{t=0}= \xi=L_\theta\varphi$ for some $ \varphi\in \mathcal H$.   The following theorem shows us how to compute such gradient.
\begin{theorem}\label{thm:general-functional} Let $\calF(q_\theta)$ be a functional depending only on the kernel mean embedding of $q_\theta$. Precisely, assume that   
$\calF(q_\theta) = H(\bmu_{q_\theta})$ for some functional $H$ with the chain rule property
\begin{equation}\label{gradient}
\partial_{\theta_i} [H(\bmu_{q_\theta})]= \langle h_\theta, \partial_{\theta_i} [\bmu_{q_\theta}]  \rangle_{\calH}  
\end{equation}
 for some function $h_\theta\in \calH$ and for all $\theta\in\Theta$. Then the gradient of $\calF$ w.r.t. the discrepancy $ d_{\alpha,\beta}$ is given by
\[
\underset{ d_{\alpha,\beta}}{\grad} \mathcal{F}(q_\theta) =L_\theta u,
\]
where 
\begin{equation}\label{u-eq}
    (\alpha D_\theta+\beta I) u = h_{\theta}.
\end{equation}
\end{theorem}

Let $\calF(q_\theta)$ be the  functional as in Theorem~\ref{thm:general-functional}, and
consider the gradient flow of $\calF(q_\theta)$ with respect to $ d_{\alpha,\beta}$. We note that this is a gradient regularized flow. According to Theorem~\ref{thm:general-functional}, the equation of this flow is given by 
\begin{equation}\label{eq-flow-1}
\partial_t \theta_t =-
\underset{ d_{\alpha,\beta}}{\grad} \mathcal{F}(q_{\theta_t}) =-L_{\theta_t} u_t,
\end{equation}
where 
\begin{equation}\label{eq-flow-2}
    (\alpha D_{\theta_t}+\beta  I) u_t = h_{\theta_t}.
\end{equation}
The following proposition shows that these gradient flows are indeed descent directions of the functional defined  on the statistical manifold $\mathcal{P}_{\Theta}$:
\begin{proposition}\label{prop:first-derivative}
 Along the gradient flow \eqref{eq-flow-1}--\eqref{eq-flow-2} of $\calF(q_\theta)$, we have 
\begin{align*}
\frac{d}{dt} \calF(q_{\theta_t})
&= -\big[\alpha \|D_{\theta_t} u_t\|_{\mathcal{H}}^2 + \beta \| L_{\theta_t} u_t\|^2 \big]\\
&=-\frac{1}{\alpha}  \Big[ \|h_{\theta_t}\|_{\mathcal{H}}^2  - \beta \scalT{h_{\theta_t}}{(\alpha  D_{\theta_t} +\beta I)^{-1} h_{\theta_t}  }_{\mathcal{H}}\Big]\\
&\leq 0,
\end{align*}
where $h_\theta$ is defined by \eqref{gradient}.
Moreover, we have $\frac{ d \mathcal{F}(q_{\theta_{t}})}{dt} <0$  if and only if  $D_{\theta_t} h_{\theta_t}\neq 0$.
\end{proposition}

\noindent \textbf{Intuition on the role of $D_{\theta}$ in the gradient flow.} The operator $D_{\theta}$ plays a central role in our framework and  we give here an intuitive interpretation of its role from the gradient flow lens. Let $\{(\lambda_j(\theta),
v_{j}(\theta))\}_{j=1}^\infty$ be the eigenvalues and  eigenfunctions of the operator $D_\theta$. Then as we have
$u_t=\sum_{j=0}^{\infty} \frac{1}{\alpha\lambda_j(\theta_t)+\beta}\scalT{h_{\theta_t}}{v_{j}(\theta_t)}_{\mathcal{H}}v_{j}(\theta_t)$,
the flow equation can be written as follows:
$\partial_t \theta_t =-
\underset{ d_{\alpha,\beta}}{\grad} \mathcal{F}(q_{\theta_t}) =-\sum_{j=0}^{\infty}  \frac{1}{\alpha\lambda_j(\theta_t) +\beta}\scalT{h_{\theta_t}}{v_{j}(\theta_t)}_{\mathcal{H}}L_{\theta_t}(v_{j}(\theta_t))$.
The eigenfunctions of $D_{\theta_t}$ provide the descent directions $L_{\theta_t}(v_{j}(\theta_t))$, that are linearly combined according to the similarity of $v_{j}(\theta_t)$ and $h_{\theta_t}$ and weighted by a factor $\nicefrac{1}{(\alpha\lambda_j(\theta_t)+\beta)}$. Small eigenvalues are noisy directions and spectral filtering them via $\alpha$ and $\beta$ favors descent directions with larger eigenvalues.

\subsection{Gradient Flows of Particular Functionals: MMD GAN as Gradient Flow}
\textbf{$\mathrm{MMD}$ GAN as a Gradient Flow w.r.t. $d_{\alpha,\beta}$.} The next result shows that the flow \eqref{theta-eq1}--\eqref{theta-eq2} of the parametric gradient regularized $\mathrm{MMD}$ GAN coincides with the gradient flow of $\calF(q_\theta)=\frac12 \mathrm{MMD}(p, q_\theta)^2$ with respect to $ d_{\alpha,\beta}$.   
\begin{corollary}\label{specific-flow} 
 For  $\calF(q_\theta)=\frac12\mathrm{MMD}(p,q_\theta)^2$,
we have 
$
\underset{ d_{\alpha,\beta}}{\grad} \mathcal{F}(q_\theta) =L_\theta u
$ with 
 $  (\alpha D_\theta+\beta I) u = -\bmu_{p-q_\theta}.$
Consequently, the parametric regularized flow \eqref{theta-eq1}--\eqref{theta-eq2}  on the statistical manifold is the gradient flow of $\calF$ with respect to $ d_{\alpha,\beta}$
\end{corollary}
\begin{proof}
 This follows immediately from Theorem~\ref{thm:general-functional} noting that
 $$\partial_{\theta_i}[\calF(q_\theta)] = -\scalT{\bmu_{p-q_{\theta}}}{ \partial_{\theta_i}[\bmu_{q_\theta}] }_{\mathcal{H}}.$$
\end{proof}

\noindent \textbf{Gradient flows of Generic Functionals on the Statistical Manifold.}
Our framework is not limited to the $\rm{MMD}$ functional. In  the following corollary we   exhibit additional examples of functional $\mathcal{F}(q_{\theta})$  defined on the statistical manifold along with their gradient flows w.r.t.  $d_{\alpha,\beta}$.
\begin{corollary}\label{specialcase}

\begin{itemize}
\item For the (potential) energy functional  $\calF_1(q_\theta)=\int V(x) q_{\theta}(dx)$
with $V \in \mathcal{H}$,
we have 
$
\underset{ d_{\alpha,\beta}}{\grad} \mathcal{F}_1(q_\theta) =L_\theta u,
$ where
   $ (\alpha D_\theta+\beta I) u = V.$
\item For the (entropy) functional $\calF_1(q_\theta)=\int f(\mu_{q_{\theta}}(x))dx$ with $f:\R \to \R$ being continuously differentiable, we have $
\underset{ d_{\alpha,\beta}}{\grad} \mathcal{F}_1(q_\theta) =L_\theta u,
$ where 
   $ (\alpha D_\theta+\beta I) u = \int f'(\bmu_{q_{\theta}}(x))k(x,.) dx.$
\item For the (interaction) functional $\mathcal{F}_3(q_{\theta})=\int  f(x)g(y)q_{\theta}(dx)q_{\theta}(dy)$ with $f,g \in \mathcal{H}$, we have $
\underset{ d_{\alpha,\beta}}{\grad} \mathcal{F}_3(q_\theta) =L_\theta u,
$ where 
    $(\alpha D_\theta+\beta I) u = \scalT{f}{\bmu_{q_\theta}}_{\calH}g+\scalT{g}{\bmu_{q_{\theta}}}_{\calH}f .$
\end{itemize}
\end{corollary}

\noindent \textbf{Continuous Time Descent } In this section we analyze the convergence properties of the kernelized parametric gradient flows of functionals defined on the statistical manifold w.r.t. to $d_{\alpha,\beta}$. The following proposition studies the convergence behavior of the parametric flows. 
\begin{proposition}[Convergence up to a barrier]\label{prop:conv}
Let $\calF(q_\theta)$ be the  functional as in Theorem~\ref{thm:general-functional}. Assume that there exists a continuous  function $\gamma:\Theta \to [0,\infty)$  such that 
\begin{equation}\label{F-vs-gradient}
\|h_\theta\|_{\calH}^2\geq \gamma(\theta) \, \calF(q_\theta)  \quad \forall \theta\in \Theta.  
\end{equation}
Consider the dynamic $t\in [0,+\infty)\to \theta_t\in\Theta$ of the gradient flow \eqref{eq-flow-1}--\eqref{eq-flow-2}. Then we have for every $t\geq 0$:
\begin{align}\label{decay}
   \calF(q_{\theta_t})\leq \calF(q_{\theta_0})
   e^{ -\int_0^t \frac{\lambda_i(\theta_s) a(\theta_s,u_s) \gamma(\theta_s)}{\alpha \lambda_i(\theta_s)a(\theta_s,u_s) +\beta} ds},
   \end{align}
where $\lambda_i(\theta)$ and $a(\theta,u)$ are defined in Theorem \ref{discrete-conv}.    
\end{proposition}
\begin{remark}
Note that condition \eqref{F-vs-gradient} is satisfied by the MMD, for $\gamma(\theta)=2$, since $h_{\theta}=-\bmu_{p-q_{\theta}}$, and hence $\|h_\theta\|_{\calH}^2= \gamma(\theta) \, \calF(q_\theta)$ .
\end{remark}

\section{Related Work}\label{sec:relatedwork}
\textbf{MMD GAN.} Since their introduction in \cite{GANoriginal}, many cost functions have been proposed  for training GANs. Related to our work is MMD GAN introduced in \cite{mmdGAN1,mmdGAN2} and later improved in \cite{li2017mmd}.\\

\noindent \textbf{Gradient Regularized GANs.} GAN's training is  notoriously known to be unstable. Wasserstein GAN 
\cite{WGAN} tackled that issue by considering $\calW_1$ as a loss function for training. Imposing lipchitizity in practice is challenging and     a gradient penalty  in the input space was introduced in  WGAN-GP \cite{gulrajani2017improved} as means to impose lipchitizity. Sobolev GAN \cite{SobolevGAN} connected this gradient penalty to advection and semi-Sobolev norms.  The work \cite{Arbel:2018} considered this input gradient regularizer in MMD GANs. 

Closely related to our study are the gradient penalties on the parameter space of the generator and the  discriminator  considered in  \cite{mescheder2017numerics}. It was demonstrated  in \cite{roth2017stabilizing,mescheder2018training}  that these gradient regularizers 
ensure stability and convergence of the min/max game. \cite{nagarajan2017gradient} showed that GANs with a variant of a gradient penalty on the parameter space of the discriminator is locally stable. Even though these works studied  the stability of the min/max game and its convergence to a saddle point, they  did not show the global convergence of the discrete gradient flow in the distributional sense which is what we  prove in this paper thanks to the newly proposed  Riemannian structure. \\

\noindent \textbf{Gradient Flows.}  Wasserstein $\mathcal{W}_2$ flows for minimizing functionals over probabilities (see  the excellent introduction  \cite{santambrogio2017euclidean}) are related to our work. These flows are non parametric and are given by the continuity equation. Recently Wasserstein flows were extended  to  statistical manifold \cite{chenoptimal}, and a kernelized Wasserstein  natural gradient flow was introduced in \cite{Arbel2020Kernelized}.  
Kernelized particle flows such as Stein Descent \cite{steindescent,Stein,duncan2019geometry}, Sobolev Descent\cite{SD} and  MMD flows \cite{arbel2019maximum} are not defined on a statistical manifold and only act on the particle level. \cite{Stein,duncan2019geometry} introduce similar Riemannian structures to ours for the Stein geometry, nevertheless they are not defined on a statistical manifold, it would be interesting to derive a dynamical form of the Stein metric  on a manifold of parametric explicit densities. \cite{barp2019minimum} studied minimum stein estimators and introduced a natural Stein gradient, it would be interesting to connect their study to a definition of an appropriate dynamic Stein  structure on a parametric manifold.

\section{Numerical Experiments}\label{sec:exp}
 \begin{figure*}[ht!]
\vspace{-1.0em}
    \begin{subfigure}[L]{0.65\textwidth}
    \centering
        \includegraphics[width=1\textwidth]{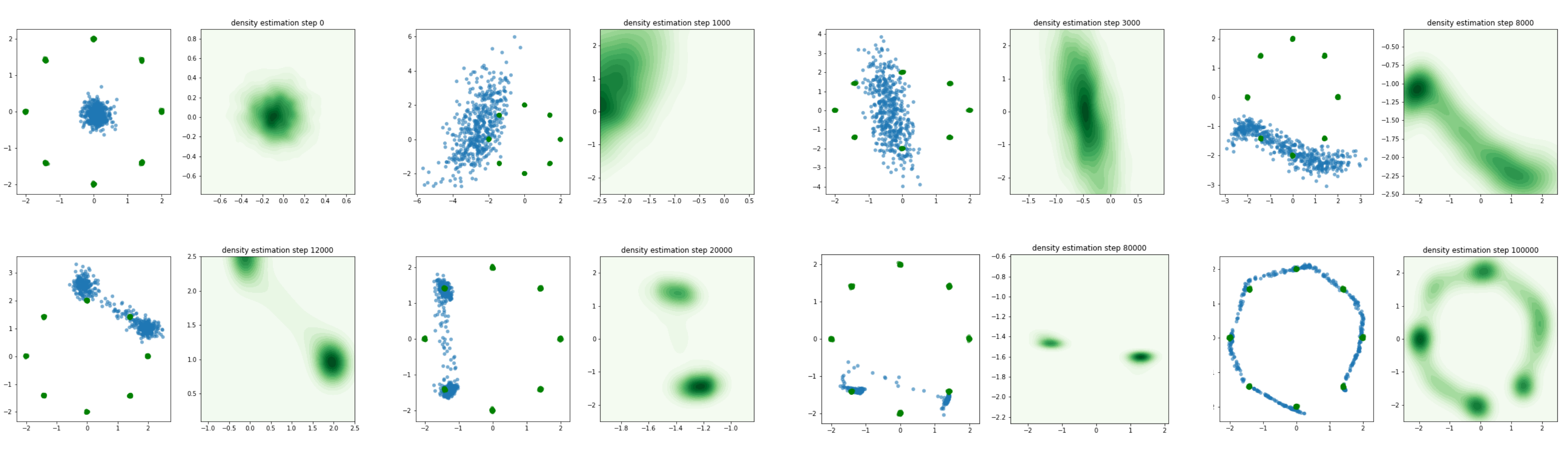}
        \caption{Trajectories of Flows for $\alpha=0$: Euclidean Gradients Flows}
        \label{fig:0traj}
    \end{subfigure}
    \begin{subfigure}[R]{0.4\textwidth}
    \centering
        \includegraphics[width=\textwidth]{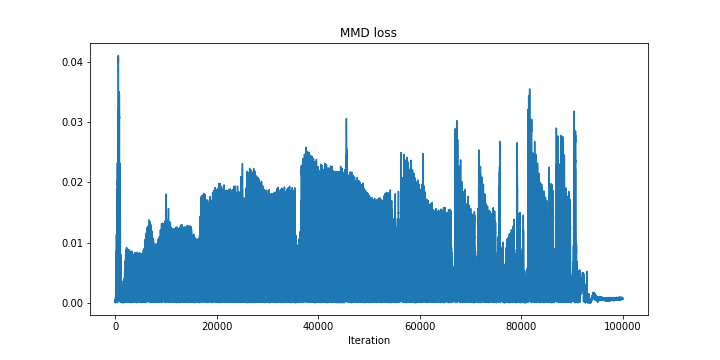}
        \caption{MMD Loss}
        \label{fig:0loss}
    \end{subfigure}
    \begin{subfigure}[L]{0.65\textwidth}
    \centering
        \includegraphics[width=\textwidth]{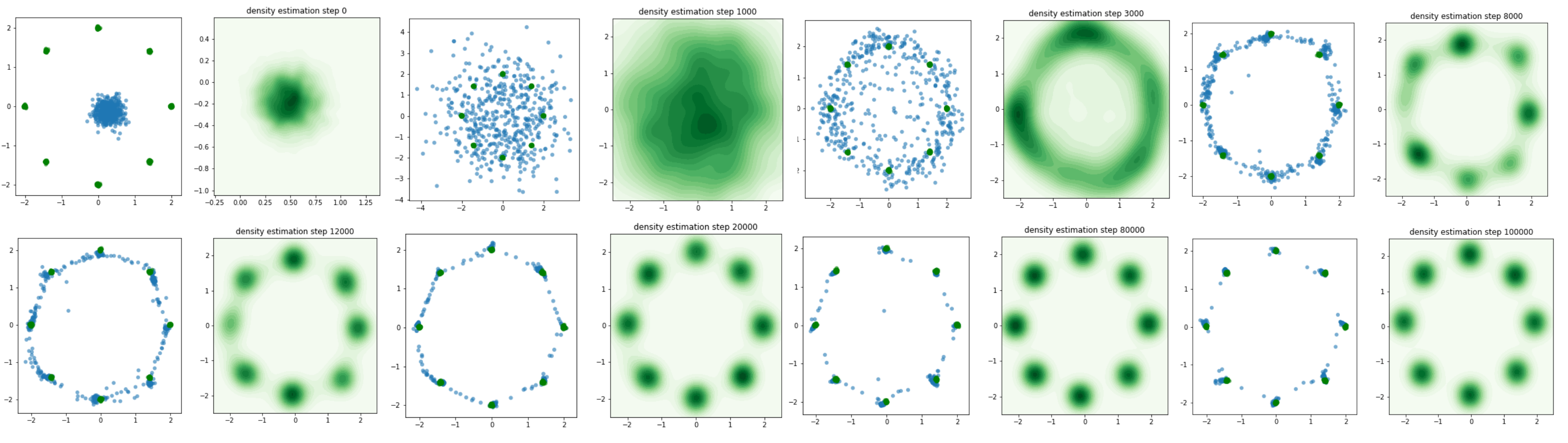}
        \caption{Trajectories of Flows for $\alpha=100$: Kernelized  Gradients Flows w.r.t. $d_{\alpha,\beta}$}
        \label{fig:100traj}
    \end{subfigure}
    \begin{subfigure}[R]{0.4\textwidth}
    \centering
        \includegraphics[width=\textwidth]{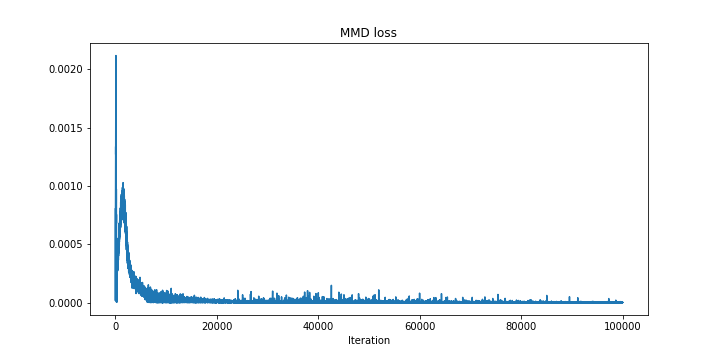}
        \caption{MMD Loss}
        \label{fig:100loss}
    \end{subfigure}
    
    \caption{ Trajectories of Kernelized Gradient flows of the $\mathrm{MMD}$ functional for $\alpha=0$ (no gradient regularization) and $\alpha=100$. It is clear that the Riemannian structure induced by  $d_{\alpha,\beta}$, $\alpha>0$ guarantees the convergence, while GAN suffer from cycles and mode collapse for $\alpha=0$. }
    \label{fig:convergence}
    \vskip -0.12 in    
\end{figure*}
 We consider a finite dimensional RKHS defined by a random feature map \cite{RF,bach2017breaking}  $\Phi(x):\Omega \to \mathbb{R}^m$. $\Phi$ in our case is a $4$ layers Relu Network with hidden dimension $m=512$, with weights sampled from  standard Gaussian and then fixed  \cite{NIPS2009_3628}. The space $\mathcal{H}=\{f(x)=\scalT{w}{\Phi(x)},\, w\in \mathbb{R}^m\}$.\\ 
 
\noindent \textbf{MMD GAN as a Kernelized Gradient Flow.}
In parametric gradient regularized  MMD GANs, we find the $\mathrm{MMD}_{\alpha,\beta}(p,q_{\theta_{\ell}})$ witness function $f_{\ell}$  between $p$ and $q_{\theta_{\ell}}$, and then update the generator according to: $\theta_{\ell+1}=\theta_{\ell}+\varepsilon_{\ell} L_{\theta_{\ell}}f_{\ell}$. 
We showed that this  coincides exactly with the gradient flow of the $\mathrm{MMD}$ functional w.r.t. the geometric structure $d_{\alpha,\beta}$ introduced in this paper, and we proved its  global convergence under mild assumptions. We show here an example, where the target distribution is a two dimensional mixture of Gaussians with $8$ modes. In this example, $\mathcal{Z}$  is $256$ dimensional space and $\nu$ is standard Gaussian on $\mathcal{Z}$. The  generator $G_{\theta}$ is a $2$ layer Relu Network  (hidden size $128$) and two dimensional output. We fix the mini-batch size in the training to $512$, the learning rate of the witness function and generator  is set to $1e^{-4}$. We train the witness function $f_{\ell}$ of $\mathrm{MMD}_{\alpha,\beta}$ with stochastic gradient (RmsProp \cite{Tieleman2012}) for $5$ iterations. Note that in this case only the last linear layer is updated in the witness function, and $\Phi$ is kept fixed.  We then update the parameter $\theta$ with gradient descent on $-\mathbb{E}_{q_{\theta_\ell}}f_{\ell}$. We use also RmsProp to ensure that the learning rates are adaptive. Note that given the witness function $f_{\ell}$, the discrete flow update given in Eq. \eqref{eq:update_theta} is exactly the gradient descent update on $-\mathbb{E}_{q_{\theta_\ell}}f_{\ell}$.\\

We give in Figure \ref{fig:convergence} the trajectories of the flows for $\alpha=100$ and $\alpha=0$ for $L=100K$ iterations. Since in our implementation (Appendix \ref{sec:algo}) we use stochastic gradient for learning the witness function $f_{\ell}$, the desired regularization effect of $\beta$ is ensured by early stopping and SGD \cite{yao2007earlystopping,bottou2016optimization}. The case $\alpha=100$ corresponds to the gradient flow w.r.t.  Riemannian structure $d_{\alpha,\beta}$, the case $\alpha=0$ corresponds to unregularized  $\mathrm{MMD}$ GAN that uses Euclidean gradients on the objective $\mathrm{MMD}(p,q_{\theta})$ \cite{mmdGAN1,mmdGAN2,li2017mmd}. We see from Figure \ref{fig:convergence} that the kernelized flows induced by $d_{\alpha,\beta}$ for $\alpha>0$ ensure graceful convergence of the generator to the target distribution as predicted by Theorem \ref{discrete-conv}, while  GAN suffers from cycles and mode collapse for $\alpha=0$.  \\

In Appendix \ref{App:plots}, we give  trajectories for neural witness functions (i.e. $\Phi$ is learned as well), and we see similar behavior to the fixed kernel. We leave for future work the extension of $d_{\alpha,\beta}$ and its gradient flows to neural function spaces as in \cite{chizat2018global,NIPS2018_8076,Ma_2018}. 
Analysing the dynamic of the  min-max optimization of parametric gradient regularized MMD GAN  descent in the neural case  will be interesting to conduct in the spirit of \cite{domingoenrich2020meanfield}.

\section{Conclusion} We propose an energy regularized gradient descent to $\mathrm{MMD}$ GAN training and derive a condition guaranteeing its global convergence.  We demonstrate that the resulting flow coincides with the gradient flow of $\mathrm{MMD}$ with respect to a newly proposed Riemannian structure on the statistical manifold of probability distributions. Our investigation deepens  the role of gradient regularization in GANs.
Future directions include more  understanding about the relationship between the static discrepancy and the dynamic discrepancy introduced in this paper, and investigation of other variational  problems on a 
statistical manifold using the proposed Riemannian structure. 

\bibliographystyle{plainnat}
\bibliography{refs,simplex}
\newpage
\appendix 

\section{Proofs}
The appendix parallels the paper in terms of sections and contains proofs of results presented within  each Section.  
\subsection{Preliminaries}
\begin{proof}[Proof of Proposition \ref{operator}]
The proof follows from noting  that $q_{\theta}= (G_{\theta})_{\#} \nu$ and $L_\theta(f) =\nabla_\theta 
\int f(x) q_\theta(dx) =\nabla_\theta[\langle f, \bmu_{q_\theta} \rangle_\calH]$.

[i)] $L_{\theta}$ and $L_{\theta}^{\top}$ are adjoint operators:
 \begin{eqnarray*}
\scalT{L_\theta(f)}{v} &=& \scalT{  \int   J_{\theta}G_{\theta}(z)\nabla f (G_{\theta}(z)) \nu(dz)}{v} \\
&=& \int v^{\top} J_{\theta}G_{\theta}(z) \nabla f (G_{\theta}(z)) \nu(dz) \\
&=&\int \scalT{\nabla f (G_{\theta}(z))}{J_{\theta}G_{\theta}(z) ^{\top}v}\nu(dz) \\
&=& \scalT{f}{\int \langle \nabla_x k(G_{\theta}(z),.), J_\theta G_{\theta}(z)^{\top}v\rangle \, \nu(d z)}_{\mathcal{H}}\\
&=& \scalT{f}{L_{\theta}^{\top}v}_{\calH}.
\end{eqnarray*}

[ii)] $D_\theta$ is symmetric, i.e., $\langle D_\theta f, g \rangle_{\mathcal H} =\langle  f, D_\theta g \rangle_{\mathcal H}$
     for $f, \, g\in \mathcal H$.
  \begin{align*}
   \langle D_\theta f, g \rangle_{\mathcal H}&= \langle L_\theta^{\top}L_{\theta} f, g \rangle_{\mathcal H}\\
     &=  \langle L_{\theta} f,   L_\theta g \rangle  \text{(applying (i))}\\
     &= \langle  f,  L_{\theta}^{\top} L_\theta g \rangle_{\calH}   \text{(applying (i) again)}\\
     &=  \langle  f, D_{\theta} g \rangle_{\calH} 
\end{align*}
[iii)] $\scalT{f}{D_\theta f}_{\mathcal{H}}=\|\nabla_\theta[\langle f, \bmu_{q_\theta} \rangle_\calH]\|^2\geq 0$. In particular, $D_\theta$ is a positive operator and hence its spectrum is contained in $[0,\infty)$.
\begin{align*}
\scalT{f}{D_\theta f}_{\mathcal{H}}&= \langle L_{\theta} f,   L_\theta f \rangle = ||  L_{\theta} f||^2 = \left|\left |\int   J_{\theta}G_{\theta}(z)\nabla f (G_{\theta}(z)) \nu(dz) \right|\right|^2 \\
&=\left|\left |\int  \nabla_\theta f(G_{\theta}(z)) \nu(dz) \right|\right|^2\\
&= \left|\left |  \scalT{f} { \int  \nabla_\theta k(G_{\theta}(z),.) \nu(dz)}_{\mathcal{H}} \right|\right|^2\\
&=  \left|\left |  \nabla_\theta \scalT{f} { \int   k(G_{\theta}(z),.) \nu(dz)}_{\mathcal{H}} \right|\right|^2 \text{ (by regularity of the kernel and linearity) }\\
&=  \left|\left |  \nabla_\theta \scalT{f} {\bmu_{q_{\theta}}}_{\mathcal{H}} \right|\right|^2
\end{align*}

[iv)] For $f\in \mathcal H$, we have $(D_\theta f)(x)= \langle D(x,\cdot), f \rangle_{\mathcal H}$ with 
$$ D(y,y') :=  \iint \langle  \partial_{\theta}k(G_{\theta}(z),y) ,  \partial_{\theta}k(G_{\theta}(z'),y') \rangle \nu(d z)\nu(d z')$$, where: $ \partial_{\theta}k(G_{\theta}(z),y) =J_{\theta}G_{\theta}(z) \nabla_{x} k(G_{\theta}(z),y)$. Equivalently $$D(y,y') =\iint  Trace (\nabla_{x} k(G_{\theta}(z),y) \otimes  \Gamma_{\theta}(z,z')  \nabla_x k(G_{\theta}(z'),y') )  \nu(d z)\nu(d z').$$
      
\begin{align*}
 \langle  f, D_{\theta} g \rangle_{\calH} &= \langle L_{\theta} f,   L_\theta g \rangle\\
 &= \scalT{ \int   J_{\theta}G_{\theta}(z)\nabla f (G_{\theta}(z)) \nu(dz) }{\int   J_{\theta}G_{\theta}(z')\nabla g (G_{\theta}(z')) \nu(dz') }\\
 &= \iint  \scalT{ J_{\theta}G_{\theta}(z)\nabla f (G_{\theta}(z)) }{    J_{\theta}G_{\theta}(z')\nabla g (G_{\theta}(z'))} \nu(dz) \nu(dz') \\
 &= \iint  \scalT{\nabla f (G_{\theta}(z)) }{ J_{\theta}G_{\theta}(z)^{\top}    J_{\theta}G_{\theta}(z')\nabla g (G_{\theta}(z'))} \nu(dz) \nu(dz') \\
  &= \iint  \scalT{\nabla f (G_{\theta}(z)) }{ \Gamma_{\theta}(z,z') \nabla g (G_{\theta}(z'))} \nu(dz) \nu(dz') \\
  &= \scalT{f}{   \iint  Trace (\nabla_{x} k(G_{\theta}(z),.) \otimes  \Gamma_{\theta}(z,z')  \nabla_x k(G_{\theta}(z'),.) )  \nu(d z)\nu(d z')  g}_{\cal{H}},
\end{align*}      
hence we identify the operator $D_{\theta}$.
\end{proof}
\subsection{Parametric Energy Regularization of MMD}
\begin{proof}[Proof of Proposition \ref{duality}]
It follows from the fact $f(x)=\langle f, k(x,\cdot)\rangle_{\mathcal H}$ that $\int f(x) p(dx) = \langle f, \bmu_{p} \rangle_{\mathcal H}$ for $f\in \mathcal H$. This and definition 
 \eqref{MMD} gives  $\mathrm{MMD}_{\alpha,\beta}(p,q_{\theta}) 
    = \sup_{f\in E_{\alpha,\beta}}   \langle f, \bmu_{p -q_\theta}\rangle_{\mathcal H}$.
 Then as the quadratic form in 
\eqref{E} is convex in $f\in \mathcal H$ and Slater’s condition holds due to $2\scalT{f}{(\alpha D_\theta + \beta I) f}_{\mathcal{H}} - 1<0$ when $f=0$,  we can apply standard  duality result \cite[Thm. 8.7.1]{Lu} to obtain
\begin{align*}
    \mathrm{MMD}_{\alpha,\beta}(p,q_{\theta}) 
    =  \min_{\gamma\geq 0} \sup_{f\in \mathcal H}  \Big\{ \langle f, \bmu_{p -q_\theta}\rangle_{\mathcal H} +\gamma \big[1- 2\scalT{f}{(\alpha D_\theta + \beta I) f}_{\mathcal{H}} \big]\Big\}.
\end{align*}
Since the above supremum is attained at the funtion $f$ satisfying $\bmu_{p -q_\theta}= 4 \gamma (\alpha D_\theta + \beta I) f$, we further get
\begin{align}\label{for-MMD}
    \mathrm{MMD}_{\alpha,\beta}(p,q_{\theta}) 
    &=\min_{\gamma\geq 0} \Big\{ \frac{1}{8\gamma} \scalT{\bmu_{p -q_\theta}}{(\alpha D_\theta  +\beta I)^{-1} \bmu_{p -q_\theta}}_{\mathcal{H}} +\gamma\Big\}\nonumber\\
    &=\sqrt{\frac12 \scalT{\bmu_{p -q_\theta}}{(\alpha D_\theta  +\beta I)^{-1} \bmu_{p -q_\theta}}_{\mathcal{H}}}.
\end{align}
By the same reasoning, we also have
\begin{align*}
    &\sup_{f\in \mathcal{H}} \left\{\int f(x)\, p(dx)- \int f(x)\,q_\theta(dx)  - \frac{\alpha}{2}  \nor{  \nabla_\theta[\langle f, \bmu_{q_\theta} \rangle_\calH] }^2 -\frac{\beta}{2} \nor{f}^2_{\mathcal{H}}\right\}\\
  &=\sup_{f\in \mathcal{H}} \langle f, \bmu_{p -q_\theta}\rangle_{\mathcal H} -\frac{1}{2} \scalT{f}{(\alpha D_\theta + \beta I) f}_{\mathcal{H}}  
  =\frac12  \scalT{\bmu_{p -q_\theta}}{(\alpha D_\theta  +\beta I)^{-1} \bmu_{p -q_\theta}}_{\mathcal{H}}
\end{align*}
and the witness function realizing this supremum is given by:
$$(\alpha D_\theta + \beta I) f^* = \bmu_{p -q_\theta}.$$
This together with \eqref{for-MMD} gives the  conclusion of the proposition.  
\end{proof}

\begin{proof}[Proof of Corollary \ref{upper-bound}]
By Proposition~\ref{duality}, we have
\begin{align*}
\mathrm{MMD}(p,q_{\theta})^2 - 2 \beta\, \mathrm{MMD}_{\alpha,\beta}(p,q_{\theta})^2=
\|\bmu_{p -q_\theta}\|_{\mathcal H}^2 -\beta \langle \bmu_{p -q_\theta}, f^* \rangle_{\calH}
=\langle \bmu_{p -q_\theta}, \bmu_{p -q_\theta} -\beta f^* \rangle_{\calH}.
\end{align*}
But  equation \eqref{witness-fn} for $f^*$ implies that $ \bmu_{p -q_\theta} -\beta f^* = \alpha D_\theta f^*$. Therefore, we obtain
\begin{align}\label{useful-iden}
    &\mathrm{MMD}(p,q_{\theta})^2 - 2 \beta\,  \mathrm{MMD}_{\alpha,\beta}(p,q_{\theta})^2
    =
\alpha \langle \bmu_{p -q_\theta}, D_\theta f^* \rangle_{\calH}\\
&=\alpha \langle (\alpha D_\theta  +\beta I) f^*, D_\theta f^* \rangle_{\calH}
=\alpha \Big[\alpha \| D_\theta f^*\|_{\calH}^2 +\beta \langle f^*, D_\theta f^* \rangle_{\calH} \Big]\geq 0.\nonumber
\end{align}
This in particular gives the  first conclusion of the corollary. When $\alpha>0$, it also shows that the equality happens if and only if $D_\theta f^*=0$. We next claim  that $D_\theta f^*=0$ is equivalent to $D_\theta \bmu_{p-q_\theta} =0$. Indeed, if $D_\theta f^*=0$ then it follows from 
\eqref{witness-fn} that $\bmu_{p-q_\theta} =\beta f^*$ and hence $D_\theta \bmu_{p-q_\theta}=\beta  D_\theta f^*=0$. Conversely, if $D_\theta \bmu_{p-q_\theta} =0$ then we obtain from \eqref{witness-fn} that 
$\alpha D_\theta(D_\theta f^*)  +\beta D_\theta f^* = 0$ yielding $\alpha \langle D_\theta(D_\theta f^*),  D_\theta f^*\rangle_{\calH}  +\beta \langle D_\theta f^* ,  D_\theta f^*\rangle_{\calH} = 0$. As the two terms are nonnegative, this only happens if $D_\theta f^*=0$. Thus we have proved the claim and the proof is complete.
\end{proof}

\subsection{Generative Adversarial Networks via Parametric Regularized Flows}
\begin{proof}[Proof of Theorem \ref{theo:statisticalManifoldDescent}]
 From the definition of $\calF$ and by using \eqref{rate-kernel-embedding}, we have 
\begin{eqnarray*}
\frac{ d \mathcal{F}(q_{\theta_{t}})}{dt}
= \scalT{\bmu_{p -q_{\theta_t}}}{ - \frac{d}{dt} \bmu_{q_{\theta_t}} }_{\mathcal{H}}
= - \scalT{\bmu_{p -q_{\theta_t}}}{D_{\theta_t} f_{t} }_{\mathcal{H}}.
\end{eqnarray*}
But from the identity \eqref{useful-iden} we obtain
\[ \scalT{\bmu_{p -q_{\theta_t}}}{ D_{\theta_t}f_{t} }_{\mathcal{H}}   = \frac{2}{\alpha}\big[ \calF(q_{\theta_t}) - \beta  \mathrm{MMD}_{\alpha,\beta}(p,q_{\theta_t})^2\big].
\]
Therefore, we conclude that the equality  in \eqref{rate} holds. On the other hand, the nonpositivity   in \eqref{rate} and the strict inequality assertion are just a consequence of Corollary~\ref{upper-bound}. 
\end{proof}

\section{Non-Asymptotic Convergence Of Gradient Descent In MMD GAN }
 \begin{proof}[Proof of Lemma \ref{compute-rate}]
 It is easy to see that
 \begin{align*}
\mathrm{MMD}( p, q)^2
=\langle \bmu_q  -\bmu_p,  \bmu_q -\bmu_p\rangle_{\calH}=  \iint  k(x,y) q(dx) q(dy) 
-2 \int \bmu_p(x) q(dx)   + \| \bmu_p\|_{\calH}^2.
\end{align*}
It follows that
\begin{align*}
 \frac{d}{d\e}
\calF(q_{\theta +\e v})
&=\frac12  \frac{d}{d\e}     \iint  k(G_{\theta+\e v}(z),G_{\theta+\e v}(z')) \nu(dz) \nu(dz') -  \frac{d}{d\e} \int \bmu_p(G_{\theta+\e v}(z)) \nu(dz)\\  
&= \frac12 \Big[   \iint \langle  \nabla_x k(G_{\theta+\e v}(z),G_{\theta+\e v}(z')),J_{\theta+\e v} G_{\theta+\e v}(z)^{\top}v \rangle  \nu(dz) \nu(dz')\\  
&\qquad + \iint  \langle  \nabla_y k(G_{\theta+\e v}(z),G_{\theta +\e v}(z')),J_{\theta +\e v} G_{\theta +\e v}(z')^{\top}v \rangle \nu(dz) \nu(dz') \Big]\\
&\qquad - \int \langle  \nabla\bmu_p(G_{\theta +\e v}(z)),J_{\theta +\e v} G_{\theta +\e v}(z)^{\top}v \rangle \nu(dz).
\end{align*}
Since $\nabla_x k(z,z') =\nabla_y k(z',z)$ due to the symmetry of $k(x,y)$, we deduce that 
\begin{align*}
  \frac{d}{d\e}
\calF(q_{\theta +\e v})
&=   \iint \langle  \nabla_x k(G_{\theta +\e v}(z),G_{\theta +\e v}(z')),J_{\theta +\e v} G_{\theta +\e v}(z)^{\top}v \rangle  \nu(dz) \nu(dz')\\
&\quad - \int \langle  \nabla\bmu_p(G_{\theta +\e v}(z)),J_{\theta + \e v} G_{\theta +\e v}(z)^{\top}v \rangle \nu(dz)\\
&= \int (L_{\theta+\e v}^{\top} v)(G_{\theta +\e v}(z')) \nu(dz')
-\langle L_{\theta +\e v} \bmu_p, v \rangle.
\end{align*}
Observe that as $L_\theta^{\top} v \in \calH$, we have $(L_\theta^{\top} v)(G_{\theta}(z')) =\big\langle k(G_{\theta}(z'), \cdot ),  L_\theta^{\top} v \big\rangle_{\calH}$ which implies that
$\int (L_\theta^{\top} v)(G_{\theta}(z')) \nu(dz') = \Big\langle \int  k(G_{\theta}(z'), \cdot ) \nu(dz'),  L_\theta^{\top} v \Big\rangle_{\calH}=\big\langle \bmu_{q_\theta},  L_\theta^{\top} v \big\rangle_{\calH}$. Thus we can rewrite the above expression as 
\begin{align*}
  \frac{d}{d\e}  
\calF(q_{\theta +\e v})
&= \big\langle \bmu_{q_{\theta+\e v}},  L_{\theta +\e v}^{\top} v \big\rangle_{\calH} -  \langle L_{\theta +\e v} \bmu_p, v \rangle =-\big\langle \bmu_{p- q_{\theta +\e v}},  L_{\theta +\e v}^{\top} v \big\rangle_{\calH}.
\end{align*}
\end{proof}

 \begin{proof}[Proof of Proposition \ref{rate0}]
 From Lemma~\ref{compute-rate}  and by our choice of $v^*$ we get 
\begin{equation*}
 \frac{d}{d\e}\Big|_{\e=0}   
\calF(q_{\theta +\e v^*})
 =- \langle  \bmu_{p-q_{\theta}}, L_\theta^{\top} L_\theta f \rangle_{\calH}=- \langle  \bmu_{p-q_{\theta}}, D_\theta f \rangle_{\calH}
 =-\Big[\alpha \| D_\theta f\|_{\calH}^2  
+\beta \langle f, D_\theta f \rangle_{\calH}\Big].
\end{equation*}
Since $D_\theta \bmu_{p-q_{\theta}} \neq 0$, we
have from the proof of Corollary~\ref{upper-bound} that $D_\theta f\neq 0$. Therefore, the above expression is negative, i.e., $v^*$ is a descent direction.
\end{proof}  

\subsection{Discrete time descent for MMD GAN}
In order to prove Theorem~\ref{discrete-conv} we need  two auxiliary lemmas. The first one is:
\begin{lemma}\label{Lipschitz-est}  Assume that $k$ and $G$ satisfy \eqref{kG-cond-1}  and   \eqref{extra-1}--\eqref{extra-2}. Then  we have for any $\theta\in\Theta$, $\e\geq 0$, and $v\in \R^p$ that
\begin{align*}
 &\|  \bmu_{q_{\theta} - q_{\theta +\e v }}\|_\calH\leq C_1  \e \|v\|,\qquad 
\big\|  L_{\theta+\e v}^\top
    v  \big\|_\calH\leq C_2 \|v\|,\\
   & \big\| ( L_{\theta+\e v}^\top
       -L_{\theta}^\top \big)v  \big\|_{\calH}
       \leq (C_3 + C_4)\e\|v\|^2,
\end{align*}
where  $C_1 := L  \E_{\nu}[D(z)]$, $C_2 := \sqrt{d p} L  \E_{\nu}[D(z)]$, $C_3 := \sqrt{p} \tilde L  \E_{\nu}[D(z)^2]$, and $C_4 := \sqrt{d} L  \E_{\nu}[\tilde D(z)]$.
\end{lemma}
\begin{proof}
We have 
  \begin{align*}
 \|  \bmu_{q_{\theta} - q_{\theta +\e v }}\|_\calH
 &=\Big\| \int \big[ k(G_{\theta}(z),\cdot) - k(G_{\theta +\e v}(z),\cdot )\big]  \nu(dz) \Big\|_\calH\\
&\leq  \int \Big\|  k(G_{\theta}(z),\cdot) - k(G_{\theta +\e v}(z),\cdot )\Big\|_\calH  \nu(dz).  \end{align*}
 Hence by using  Lipschitz condition \eqref{kG-cond-1}, we  get
 \begin{align*}
 \|  \bmu_{q_{\theta} - q_{\theta +\e v }}\|_\calH&\leq  L \e \|v\| \int D(z)  \nu(dz) =: C_1  \e \|v\|. 
 \end{align*}
 
 In the following calculations, we use the observation that  if $h=(h_1,...,h_d)\in \calH^d$ and  $b\in \R^d$ then 
\[
\| \langle h, b\rangle \|_\calH=  \| \sum_{i=1}^d h_i b_i\|_\calH \leq \sum_{i=1}^d   \|h_i\|_\calH |b_i| \leq \big(\sum_{i=1}^d   \|h_i\|_\calH^2\big)^\frac12 \big(\sum_{i=1}^d   b_i^2\big)^\frac12= \|h\|_\calH \|b\|.
\]
Also note that condition \eqref{kG-cond-1} implies that 
\[
\| \nabla_x k(x,.)\|_\calH\leq \sqrt{d} \, L \quad \mbox{and}\quad 
\|J_{\theta} G_{\theta}(z)\| \leq \sqrt{p} \, D(z).
\]
We next have 
\begin{align*}
   \big\|  L_{\theta+\e v}^\top
       v  \big\|_{\calH}
       &=\Big\| \int \langle \nabla_x k(G_{\theta+\e v }(z),.), J_{\theta +\e v} G_{\theta+\e v }(z)^{\top}v\rangle \nu(d z)\Big\|_{\calH}\\
       &\leq  \int  \Big\|\langle \nabla_x k(G_{\theta+\e v }(z),.), J_{\theta +\e v} G_{\theta+\e v }(z)^{\top}v\rangle\Big\|_{\calH} \nu(d z)
       \\
       &\leq  \|v\| \int  \| \nabla_x k(G_{\theta+\e v }(z),.)\|_\calH \|J_{\theta +\e v} G_{\theta+\e v }(z)^{\top}\|  \nu(d z)
       \\
       &\leq \sqrt{d p }  \, L \|v\| \int D(z)   \nu(d z)=: C_2 \|v\|.
\end{align*}
Finally, observe that
\begin{align*}
   &\big\| ( L_{\theta+\e v}^\top
       -L_{\theta}^\top \big)v  \big\|_{\calH}
       =\Big\| \int A_\theta(z,\cdot) \nu(d z)\Big\|_{\calH} \leq \int \|A_\theta(z,\cdot)\|_{\calH} \nu(d z),
\end{align*}
where 
\begin{align*}
 A_\theta(z, \cdot) 
 &:=\langle \nabla_x k(G_{\theta+\e v }(z),.), J_{\theta +\e v} G_{\theta+\e v }(z)^{\top}v\rangle -   \langle \nabla_x k(G_{\theta}(z),.), J_\theta G_{\theta}(z)^{\top}v\rangle \\
 &= \langle \big[\nabla_x k(G_{\theta+\e v }(z),.) - \nabla_x k(G_{\theta}(z),.)\big], J_{\theta +\e v} G_{\theta+\e v }(z)^{\top}v\rangle \\
 &\quad +\langle \nabla_x k(G_{\theta}(z),.), \big[J_{\theta +\e v} G_{\theta+\e v }(z)^{\top} -J_\theta G_{\theta}(z)^{\top}\big]v\rangle.
\end{align*}
It follows that 
$\big\| ( L_{\theta+\e v}^\top
       -L_{\theta}^\top \big)v  \big\|_{\calH}
        \leq  I_1 + I_2$
with 
\begin{align*}
I_1 
&\leq  \|v\|\int  \| \nabla_x k(G_{\theta+\e v }(z),.) - \nabla_x k(G_{\theta}(z),.)\|_\calH  \|J_{\theta +\e v} G_{\theta+\e v }(z)^{\top} \|  \nu(dz)\\
&\leq  \sqrt{p} \tilde L \|v\|\int  \|G_{\theta+\e v }(z)- G_{\theta}(z)\|  D(z) \nu(dz)
\leq \sqrt{p} \tilde L  \e \|v\|^2 \int   D(z)^2 \nu(dz)=: C_3 \e \|v\|^2    
\end{align*}
and
\begin{align*}
I_2 
&\leq  \|v\| \int  \| \nabla_x k(G_{\theta}(z),.)\|_\calH  \big\|J_{\theta +\e v} G_{\theta+\e v }(z)^{\top} -J_\theta G_{\theta}(z)^{\top}\big\|  \nu(dz)\\
&\leq \sqrt{d} L \e \|v\|^2 \int  \tilde D(z) \nu(dz) 
=: C_4 \e \|v\|^2.    
\end{align*}
Notice that we have used condition \eqref{extra-1} to estimate $I_1$ and condition \eqref{extra-2} to estimate $I_2$. 
By combining the above estimates, we obtain the desired result.
\end{proof}

The second one is:

\begin{lemma}\label{quantity-est} 
Under the assumptions of Theorem~\ref{discrete-conv},  we have  
  \begin{align}\label{second-difference}
   \calF(q_{\theta_{\ell+1}})  -\calF(q_{\theta_\ell })
   \leq - \frac12 \e_\ell \langle  \bmu_{p-q_{\theta}}, D_{\theta_\ell} f_\ell 
  \rangle_{\calH}\quad \mbox{for all }\ell\geq 1.
  \end{align}
\end{lemma}
\begin{proof}
Let  $\ell\geq 1$ be arbitrary. For convenience, let us write $\theta := \theta_\ell$, $v := L_{\theta_{\ell}} f_{\ell}$, $\theta' := \theta_{\ell}+ \varepsilon_{\ell}v$, and $f:=f_\ell$. 
 Then by  Lemma~\ref{compute-rate}, we have
\begin{equation*}
 \frac{d}{d\e}
\calF(q_{\theta +\e v })
 =- \langle  \bmu_{p-q_{\theta+\e v }}, L_{\theta+\e v}^\top v \rangle_{\calH} \quad \forall \e\in [0, \e_\ell].
 \end{equation*}
 Therefore,
 \begin{align}\label{first-difference}
   \calF(q_{\theta'})  -\calF(q_{\theta })
  & =\int_0^{\e_\ell} \frac{d}{d\e}\big[
\calF(q_{\theta +\e v })\big]\, d\e
= -\int_0^{\e_\ell} \langle  \bmu_{p-q_{\theta+\e v }}, L_{\theta+\e v}^\top v \rangle_{\calH} d\e \nonumber \\
&= -\e_\ell \langle  \bmu_{p-q_{\theta}}, L_{\theta}^\top v \rangle_{\calH} 
-\int_0^{\e_\ell} \Big[ \langle  \bmu_{p-q_{\theta +\e v}}, L_{\theta+\e v}^\top v \rangle_{\calH} -
 \langle  \bmu_{p-q_{\theta}}, L_{\theta}^\top v\rangle_{\calH} \Big] d\e. 
 \end{align}
 To estimate  the integral term,  observe that
 \begin{align*}
 \Big|\langle  \bmu_{p-q_{\theta+\e v }}, L_{\theta+\e v}^\top v \rangle_{\calH} 
 &-\langle  \bmu_{p-q_{\theta}}, L_{\theta}^\top v\rangle_{\calH}\Big|\\
 &= \Big|\langle  \bmu_{q_{\theta} - q_{\theta +\e v }}, L_{\theta +\e v}^\top  v  \rangle_{\calH}
 +\langle  \bmu_{p-q_{\theta}}, (L_{\theta +\e v}^\top -L_{\theta}^\top) v\rangle_{\calH}\Big|\\
 &\leq \|  \bmu_{q_{\theta} - q_{\theta +\e v }}\|_\calH 
 \big\|  L_{\theta+\e v}^\top
    v  \big\|_\calH
    + \|\bmu_{p-q_{\theta}}\|_\calH \big\| ( L_{\theta+\e v}^\top
    -L_{\theta}^\top \big)v  \big\|_\calH.
 \end{align*}
This together with Lemma~\ref{Lipschitz-est} 
and the definition of $\calF(q_\theta)$ gives
 \begin{align*}
 \Big| \langle  \bmu_{p-q_{\theta+\e v }}, L_{\theta+\e v}^\top v \rangle_{\calH} 
 &-\langle  \bmu_{p-q_{\theta}}, L_{\theta}^\top v\rangle_{\calH}\Big| \leq   C \e   \Big(1 + \sqrt{\calF(q_\theta)}\Big) \|v\|^2
 \end{align*}
 with $C>0$ depending only on the constants $C_1, \, C_2, \, C_3, \, C_4$ given in Lemma~\ref{Lipschitz-est}. 
 Then by combining with \eqref{first-difference} and  $L_{\theta}^\top v =D_\theta f$ we arrive at 
 \begin{align*}
   \calF(q_{\theta'})  -\calF(q_{\theta })
  &\leq  -\e_\ell \langle  \bmu_{p-q_{\theta}}, D_{\theta} f 
  \rangle_{\calH} +C 2^{-1} \e_\ell^2   \Big(1 + \sqrt{\calF(q_\theta)}\Big) \|v\|^2. 
  \end{align*}
 Since  $\langle  \bmu_{p-q_{\theta}}, D_{\theta} f 
  \rangle_{\calH}=\langle  (\alpha D_\theta +\beta I) f, D_{\theta} f 
  \rangle_{\calH}=\alpha \|D_\theta f\|_\calH^2 + \beta \langle f, D_{\theta} f \rangle_\calH $ and
  $\|v\|^2 = \langle f, D_{\theta} f \rangle_\calH$ by iii) of Proposition~\ref{operator}, we also have  $\|v\|^2  \leq \frac{1}{\beta} \langle \bmu_{p-q_{\theta}}, D_{\theta} f_\ell \rangle_\calH$. Thus  we infer further that
 \begin{align}\label{key-est}
   \calF(q_{\theta_{\ell+1}})  -\calF(q_{\theta_\ell })
  &\leq  -\Big[1- C (2\beta)^{-1}\e_\ell \Big(1 + \sqrt{\calF(q_{\theta_\ell})}\Big) 
  \Big]\e_\ell \langle  \bmu_{p-q_{\theta_\ell}}, D_{\theta_\ell} f_\ell \rangle_{\calH}\quad \forall \ell\geq 1,
  \end{align}
where we have returned to the original notation. We claim that \eqref{key-est} implies in particular that 
 $\calF(q_{\theta_\ell })\leq \calF(q_{\theta_1 }) $ for all $\ell\geq 1$. Indeed,  from \eqref{key-est} for $\ell=1$ and condition \eqref{step-size-cond} we get $\calF(q_{\theta_{2}})  -\calF(q_{\theta_1 })\leq 0$. Now if $\calF(q_{\theta_\ell })\leq \calF(q_{\theta_1 }) $ for some $\ell\geq 1$, then again \eqref{key-est}  and condition \eqref{step-size-cond} give
\[
\calF(q_{\theta_{\ell+1}})  -\calF(q_{\theta_\ell })
  \leq  -\Big[1- C (2\beta)^{-1}\e_\ell \Big(1 + \sqrt{\calF(q_{\theta_1})}\Big) \Big]\e_\ell \langle  \bmu_{p-q_{\theta_\ell}}, D_{\theta_\ell} f_\ell \rangle_{\calH}\leq 0.
\]
Thus,  $\calF(q_{\theta_{\ell+1}})  \leq \calF(q_{\theta_\ell })\leq \calF(q_{\theta_1 })$. Therefore, it follows by induction that the claim holds true. Owing to this claim and condition \eqref{step-size-cond}, we deduce from \eqref{key-est} that 
\begin{align*}
   \calF(q_{\theta_{\ell+1}})  -\calF(q_{\theta_\ell })
  &\leq  -\Big[1- C (2\beta)^{-1}\e_\ell \Big(1 + \sqrt{\calF(q_{\theta_1})}\Big) \Big]\e_\ell \langle  \bmu_{p-q_{\theta_\ell}}, D_{\theta_\ell} f_\ell \rangle_{\calH}\\
  &\leq  -\frac12  \e_\ell \langle  \bmu_{p-q_{\theta}}, D_{\theta_\ell} f_\ell 
  \rangle_{\calH}
  \end{align*}
  for every $\ell\geq 1$.
\end{proof}

\begin{proof}[{\bf Proof of Theorem~\ref{discrete-conv}}] In order to not to clutter notations we note hereafter $\alpha:=\alpha_{\ell}$, and $\beta:= \beta_{\ell}$.
We have 
 $\alpha \langle \bmu_{p -q_{\theta_\ell}}, D_{\theta_\ell} f_\ell \rangle_{\calH} =2\calF(q_{\theta_\ell}) - 2 \beta\,  \mathrm{MMD}_{\alpha,\beta}(p,q_{\theta_\ell})^2$ from \eqref{useful-iden}. Thus by combining with  Lemma~\ref{quantity-est} we obtain 
 \begin{align*}
   \frac{\calF(q_{\theta_{\ell+1}})  -\calF(q_{\theta_\ell })}{\e_\ell}    \leq  -\alpha^{-1} \calF(q_{\theta_\ell}) + \beta \alpha^{-1} \mathrm{MMD}_{\alpha,\beta}(p,q_{\theta_\ell})^2
   \quad \forall \ell\geq 1.
  \end{align*}
  Let $\{\lambda_j,d_j\}_{j=0\dots \infty}$ be the eigensystem of $D_{\theta_{\ell}}$ , the eigenvalues are given in decreasing order. Let $\lambda_{i}(\theta)$ be the smallest non zero eigenvalue. 
 \begin{eqnarray*} 
 2  \mathrm{MMD}_{\alpha,\beta}(p,q_{\theta_\ell})^2 =\scalT{\bmu_{p- q_{\theta_{\ell}}}}{f_{\ell}  }_{\mathcal{H}}&=& \langle (\alpha D_{\theta_t} +\beta I) f_{\ell} ,f_{\ell} \rangle_{\calH}\\
 &=& \alpha \sum_{j=0}^{\infty} \lambda_j \scalT{f_{\ell}}{d_j}^2 + \beta ||f_{\ell}||^2_{\calH} \\
 &=& \alpha \sum_{j=0}^{i}  \lambda_j \scalT{f_{\ell}}{d_j}^2  + \beta ||f_{\ell}||^2_{\calH}\\
 &\geq&   \alpha \lambda_{i}(\theta_{\ell})  ||f_{\ell}||^2_{\calH} \left( 1 - \frac{ || P_{\mathrm{Null}(D_{\theta_{\ell}})}f_{\ell} ||^2_{\calH}}{ ||f_{\ell}||^2_{\calH} }\right) +  \beta ||f_{\ell}||^2_{\calH},
 \end{eqnarray*}
 where the projection on the null space of $D_{\theta_{\ell}}$ : 
 $$ P_{\mathrm{Null}(D_{\theta_{\ell}})}f_{\ell} = \sum_{j>i } \scalT{d_j}{f_{\ell}} d_j.$$
 Let $ a(\theta_{\ell}, f_{\ell})=  1 - \frac{ || P_{\mathrm{Null}(D_{\theta_{\ell}})}f_{\ell} ||^2_{\calH}}{ ||f_{\ell}||^2_{\calH} }$.  By our choice of $\theta_1$  and $\varepsilon_{\ell-1}$, we have $f_{\ell} \notin \mathrm{Null}(D_{\theta_{\ell}}) $, since we choose $\varepsilon_{\ell-1}$ so that:
 \[a(\theta_{\ell},f_{\ell})>\tau >0\]
 We conclude that:
 \[ \scalT{\bmu_{p- q_{\theta_{\ell}}}}{f_{\ell}  }_{\mathcal{H}} \geq \nor{f_{\ell}}^2_{\mathcal{H}} \left( \alpha \lambda_{i}(\theta_{\ell}) a(\theta_{\ell},f_{\ell}) + \beta \right)  \]
 
  Using this  we obtain
\[
\scalT{\bmu_{p- q_{\theta_{\ell}}}}{f_{\ell}  }_{\mathcal{H}}\leq \| \bmu_{p- q_{\theta_{\ell}}}\|_{\calH} \|f_{\ell}\|_{\calH}
\leq  \| \bmu_{p- q_{\theta_{\ell}}}\|_{\calH} \sqrt{\frac{\scalT{\bmu_{p- q_{\theta_{\ell}}}}{f_{\ell}  }_{\mathcal{H}}}{ \alpha \lambda_{i}(\theta_{\ell}) a(\theta_{\ell},f_{\ell}) + \beta}}
\]
yielding 
\[
\scalT{\bmu_{p- q_{\theta_{\ell}}}}{f_{\ell}  }_{\mathcal{H}}\leq\frac{1}{ \alpha \lambda_{i}(\theta_{\ell}) a(\theta_{\ell},f_{\ell}) + \beta}  \| \bmu_{p- q_{\theta_{\ell}}}\|^2_{\calH}.
\] 
  
  
Hence it follows that: 
\[ \mathrm{MMD}_{\alpha,\beta}(p,q_{\theta_\ell})^2 \leq  \frac{ \calF(q_{\theta_\ell })}{ \alpha \lambda_{i}(\theta_{\ell}) a(\theta_{\ell},f_{\ell}) + \beta }.
\]
Thus 
\begin{align*}
  \frac{ \calF(q_{\theta_{\ell+1}})  -\calF(q_{\theta_\ell })}{\e_\ell}     \leq  -\alpha^{-1}  \Big( 1- \frac{\beta}{\alpha \lambda_i(\theta_\ell) a(\theta_{\ell},f_{\ell})+\beta}\Big) \calF(q_{\theta_\ell})
   =- \frac{  \lambda_i(\theta_\ell) a(\theta_{\ell},f_{\ell})}{\alpha \lambda_i(\theta_\ell) a(\theta_{\ell},f_{\ell}) +\beta} \calF(q_{\theta_\ell})
  \end{align*}
  which implies 
  \begin{align*}
   \calF(q_{\theta_{\ell+1}})    \leq  \Big[1 - \frac{ \e_\ell \lambda_i(\theta_\ell)a(\theta_{\ell},f_{\ell})}{\alpha \lambda_i(\theta_\ell)a(\theta_{\ell},f_{\ell}) +\beta}\Big] \calF(q_{\theta_\ell})
  \end{align*}
  for every  $\ell=1,2,...$ By iterating this estimate, we arrive at
  \begin{align*}
    \calF(q_{\theta_{\ell +1}})   
   & \leq  \calF(q_{\theta_1 }) \prod_{j=1}^\ell\Big[1 - \frac{ \e_j \lambda_i(\theta_j)a(\theta_{j},f_{j})}{\alpha_j \lambda_i(\theta_j)a(\theta_{j},f_{j}) +\beta_j}\Big]  \quad \forall \ell\geq 1.\\
   &= \calF(q_{\theta_1 }) \prod_{j=1}^\ell\Big[1 - \e_j\chi_j\Big]\\
   &\leq \calF(q_{\theta_1 })  \exp(-\sum_{j=1}^{\ell}\e_j\chi_j).
  \end{align*}
  In particular we have for  for $\alpha_{\ell}\leq \frac{ \tau}{2}$, and $\beta_{\ell}=\alpha_{\ell}\lambda_{i}(\theta_{\ell})$, we have noting that for all $j$, we have: $\tau < a(\theta_j,f_j)  \leq 1$:
  \[\chi_j= \frac{ \lambda_i(\theta_j)a(\theta_{j},f_{j})}{\alpha_j \lambda_i(\theta_j)a(\theta_{j},f_{j}) +\beta_j} \geq \frac{\lambda_{i}(\theta) \tau}{2\alpha_j \lambda_i(\theta)}= \frac{\tau}{2\alpha_j} \geq 1,\]
  and hence we have:
  $$ \calF(q_{\theta_{\ell +1}})   \leq \calF(q_{\theta_1 })  \exp(-\sum_{j=1}^{\ell}\e_j). $$

\end{proof}

\section{Parametric Kernelized Flows for a General Functional}
\subsection{Dynamic Formulation For \texorpdfstring{$\mathrm{MMD}$}{Lg} on a Statistical Manifold}

\begin{proof}[Proof of Theorem \ref{dynamic}]
By using the dynamic of $\bmu_{q_{\theta_t}}$ in \eqref{rate-kernel-embedding}, we have 
\begin{equation*}
\mathrm{MMD}(q_{\theta_0},q_{\theta_1})=
\|\bmu_{q_{\theta_1}} - \bmu_{q_{\theta_0}}\|_{\mathcal{H}}
 =\nor{\int_{0}^1 \frac{d  }{dt} \bmu_{q_{\theta_t}} dt}_{\mathcal{H}}
= \nor{\int_{0}^1 D_{\theta_t} f_{t} dt}_{\mathcal{H}}
\end{equation*}
which together with Jensen inequality yields
$$\mathrm{MMD}^2(q_{\theta_0},q_{\theta_1})= \nor{\int_{0}^1 D_{\theta_t} f_{t} dt}^2_{\mathcal{H}}\leq \int_{0}^1 \nor{D_{\theta_t} f_{t}}^2_{\mathcal{H}} dt.$$
Thus it is enough to prove that there exists an admissible  path $(\theta^*_{t},f^*_{t})$ satisfying
\begin{equation}\label{minimizer}
    \mathrm{MMD}^2(q_{\theta_0},q_{\theta_1})= \int_{0}^1 \nor{D_{\theta^*_t} f^*_{t}}^2_{\mathcal{H}} dt.
\end{equation}
Let $(\theta^*_{t},f^*_{t})$ be the path given by assumption \eqref{eq:assump}. In particular, we have 
$$ D_{\theta^*_t} f^*_t= 2 \bmu_{q_{\theta^*_1}- q_{\theta^*_{t}}}.$$
This together with a calculation in  the proof of Theorem \ref{theo:statisticalManifoldDescent} gives
$$\frac{d}{dt} \mathrm{MMD}^2(q_{\theta^*_{t}},q_{\theta^*_1}) =  - 2 \scalT{\bmu_{q_{\theta^*_1} - q_{\theta^*_t}}}{ D_{\theta^*_t} f^*_{t}}_{\mathcal{H}} = -\nor{D_{\theta^*_t} f^*_t}^2_{\mathcal{H}}.$$
It follows that this  path satisfies optimal property \eqref{minimizer} since
$$\mathrm{MMD}^2(q_{\theta_0},q_{\theta_1})= \int_{0}^1 -\frac{d}{dt} \mathrm{MMD}^2(q_{\theta^*_{t}},q_{\theta^*_1}) dt = \int_{0}^1 \nor{D_{\theta^*_t} f^*_t}^2_{\mathcal{H}} dt. $$
\end{proof}
\subsection{Regularized \texorpdfstring{$\mathrm{MMD}$}{Lg} and gradient flows on a Statistical Manifold}

\begin{proof}[Proof of Theorem \ref{thm:general-functional}]
Let  $t\mapsto \theta_t\in \Theta$ be  a differentiable curve  passing through $\theta$ at $t=0$ and with tangent vector
\begin{equation}\label{initial-tangent}
 \partial_t \theta_t\big|_{t=0}= \xi=L_\theta\varphi\quad\mbox{for } \varphi\in \mathcal H.
\end{equation}
Then by using the chain rule we obtain
\begin{eqnarray*}
 \frac{d \calF(q_{\theta_t})}{dt}
\Big|_{t=0} 
&=& \scalT{\nabla_{\theta}[\mathcal{F}(q_{\theta_t})]}{\frac{d\theta_t}{dt}}_{\R^p}\Big|_{t=0}
=\scalT{\langle h_{\theta_t}, \nabla_{\theta} [\bmu_{q_{\theta_t}}]  \rangle_{\calH}  }{\frac{d\theta_t}{dt}}_{\R^p}\Big|_{t=0}\\
&=&\Big\langle h_{\theta_t}, \nabla_{\theta} [\bmu_{q_{\theta_t}}] \frac{d\theta_t}{dt} \Big\rangle_{\calH}\Big|_{t=0}
=\langle h_{\theta_t}, \frac{d}{dt} [\bmu_{q_{\theta_t}}] \rangle_{\calH}\Big|_{t=0}.
 \end{eqnarray*}
Hence it follows from \eqref{initial-tangent}, 
\eqref{rate-kernel-embedding}, and \eqref{u-eq} that
\begin{eqnarray*}
\frac{d \calF(q_{\theta_t})}{dt}
\Big|_{t=0}
= \langle h_{\theta}, D_\theta \varphi\rangle_{\calH}=\scalT{(\alpha D_\theta +\beta I) u}{ D_\theta \varphi}
= g_{\theta}\left(L_\theta u, \xi\right)_{\mathcal H}.
 \end{eqnarray*}
 Therefore, we conclude from definition \eqref{grad-def} that $\underset{ d_{\alpha,\beta}}{\grad} \mathcal{F}(q_{\theta}) =L_\theta u$.
\end{proof}
\textbf{Gradient flows of Generic Functionals on the Statistical Manifold.}
Our framework is not limited to the $\rm{MMD}$ functional. In  the following corollary we   exhibit additional examples of functional $\mathcal{F}(q_{\theta})$  defined on the statistical manifold along with their gradient flows w.r.t.  $d_{\alpha,\beta}$.
\begin{corollary}\label{specialcase}

\begin{itemize}
\item For the (potential) energy functional  $\calF_1(q_\theta)=\int V(x) q_{\theta}(dx)$
with $V \in \mathcal{H}$,
we have 
$
\underset{ d_{\alpha,\beta}}{\grad} \mathcal{F}_1(q_\theta) =L_\theta u,
$ where
   $ (\alpha D_\theta+\beta I) u = V.$
\item For the (entropy) functional $\calF_1(q_\theta)=\int f(\mu_{q_{\theta}}(x))dx$ with $f:\R \to \R$ being continuously differentiable, we have $
\underset{ d_{\alpha,\beta}}{\grad} \mathcal{F}_1(q_\theta) =L_\theta u,
$ where 
   $ (\alpha D_\theta+\beta I) u = \int f'(\bmu_{q_{\theta}}(x))k(x,.) dx.$
\item For the (interaction) functional $\mathcal{F}_3(q_{\theta})=\int  f(x)g(y)q_{\theta}(dx)q_{\theta}(dy)$ with $f,g \in \mathcal{H}$, we have $
\underset{ d_{\alpha,\beta}}{\grad} \mathcal{F}_3(q_\theta) =L_\theta u,
$ where 
    $(\alpha D_\theta+\beta I) u = \scalT{f}{\bmu_{q_\theta}}_{\calH}g+\scalT{g}{\bmu_{q_{\theta}}}_{\calH}f .$
\end{itemize}
\end{corollary}

\begin{proof}[Proof of Corollary \ref{specialcase}]
 This is a consequence of Theorem~\ref{thm:general-functional} observing  that
 $\partial_{\theta_i}[\calF_1(q_\theta)] = \scalT{V}{ \partial_{\theta_i}[\bmu_{q_\theta}] }_{\mathcal{H}}$
 since $\calF_1(q_\theta)=\langle V, \bmu_{q_\theta}\rangle_{\mathcal H}$ and  $\partial_{\theta_i}\calF_2(q_\theta) = \scalT{\int f'(\bmu_{q_{\theta}}(x))k(x,.) dx}{\partial_{\theta_i}\bmu_{q_{\theta}}}_{\mathcal{H}}$ since 
 $\calF_2(q_\theta) =\int f\left(\scalT{\bmu_{q_{\theta}}}{k(x,.)}_{\mathcal{H}}\right)dx$. Note also that as  $\mathcal{F}_3(q_{\theta})=\scalT{f}{\bmu_{q_{\theta}}}_{\mathcal{H}}\scalT{g}{\bmu_{q_{\theta}}}_{\mathcal{H}}$, we have  \[\partial_{\theta_i}[\calF_3(q_\theta)] = \scalT{\scalT{f}{\bmu_{q_\theta}}_{\calH}g+\scalT{g}{\bmu_{q_{\theta}}}_{\calH}f}{ \partial_{\theta_i}[\bmu_{q_\theta}] }_{\mathcal{H}}.
 \]
\end{proof}

\begin{proof}[Proof of Proposition \ref{prop:first-derivative}]
This general fact can be seen from the proof of Theorem~\ref{thm:general-functional}. Indeed, 
 \begin{align}\label{gen-derivative}
\frac{d}{dt} \calF(q_{\theta_t})
&= g_{\theta_t}(\underset{ d_{\alpha,\beta}}{\grad} \mathcal{F}(q_{\theta_t}), \partial_t \theta_t)
 =-g_{\theta_t}\big(L_{\theta_t} u_t, L_{\theta_t} u_t\big)\nonumber\\
 &=- \scalT{ (\alpha  D_{\theta_t} +\beta I) u_t}{D_{\theta_t} u_t}_{\mathcal{H}}=- \scalT{h_{\theta_t}}{D_{\theta_t} u_t}_{\mathcal{H}}.
     \end{align}
On the other hand, we obtain  from  equation \eqref{eq-flow-2} for $u_t$ that
\[\alpha \scalT{h_{\theta_t}}{ D_{\theta_t}u_{t} }_{\mathcal{H}}  +\beta \scalT{h_{\theta_t}}{u_{t} }_{\mathcal{H}} =\|h_{\theta_t}\|_{\mathcal{H}}^2
\]
which gives 
\[ \scalT{h_{\theta_t}}{D_{\theta_t} u_t}_{\mathcal{H}}  =\frac{1}{\alpha}  \Big[ \|h_{\theta_t}\|_{\mathcal{H}}^2  - \beta \scalT{h_{\theta_t}}{u_{t} }_{\mathcal{H}}\Big]
=\frac{1}{\alpha}  \Big[ \|h_{\theta_t}\|_{\mathcal{H}}^2  - \beta \scalT{h_{\theta_t}}{(\alpha  D_{\theta_t} +\beta I)^{-1} h_{\theta_t}  }_{\mathcal{H}}\Big].
\]
Therefore, we deduce the first conclusion of the proposition. From  \eqref{gen-derivative}, we also see that $\frac{d}{dt} \calF(q_{\theta_t}) =0$ if and only if $\alpha  \|D_{\theta_t}u_t\|_{\calH}^2 +\beta \scalT{  u_t}{D_{\theta_t} u_t}_{\mathcal{H}}=0$. That is, $\frac{d}{dt} \calF(q_{\theta_t}) =0$ if and only if $D_{\theta_t}u_t=0$. But by exactly the same reason as in the proof of Corollary~\ref{upper-bound}, we have $D_{\theta_t}u_t=0$ is equivalent to $D_{\theta_t} h_{\theta_t}=0$. Thus the last conclusion of the proposition follows.
\end{proof}

\section{Convergence of Flows}
\begin{proof}[Proof of Proposition \ref{prop:conv}]
From Proposition~\ref{prop:first-derivative}, we have 
 \begin{align}\label{derivative-F}
\frac{d}{dt} \calF(q_{\theta_t})
&=-\frac{1}{\alpha}  \Big[ \|h_{\theta_t}\|_{\mathcal{H}}^2  - \beta \scalT{h_{\theta_t}}{u_t  }_{\mathcal{H}}\Big].
\end{align}
By the proof of Theorem~\ref{discrete-conv} we have:
\[
\scalT{h_{\theta_t}}{u_t  }_{\mathcal{H}}\leq\frac{1}{\alpha \lambda_i(\theta_t)a(\theta_t,u_t) + \beta} \|h_{\theta_t}\|_{\calH}^2.
\]
This together with \eqref{derivative-F} and assumption \eqref{F-vs-gradient} gives
\begin{align*}
\frac{d}{dt} \calF(q_{\theta_t})
&\leq -\frac{1}{\alpha}  \Big( 1  - \frac{\beta}{\alpha \lambda_i(\theta_t)a(\theta_t,u_t) + \beta} \Big)\|h_{\theta_t}\|_{\calH}^2
\leq - \frac{\lambda_i(\theta_t)a(\theta_t,u_t)  \gamma(\theta_t)}{\alpha \lambda_i(\theta_t) a(\theta_t,u_t)+ \beta} 
 \calF(q_{\theta_t}).
\end{align*}
It follows that
\begin{align*}
    \log{\Big(\frac{\calF(q_{\theta_t})}{\calF(q_{\theta_0})}\Big)}= \int_0^t \frac{d}{ds} \big[ \log(\calF(q_{\theta_t}))\big] ds
    \leq  -\int_0^t \frac{\lambda_i(\theta_s)a(\theta_s,u_s) \gamma(\theta_s)}{\alpha \lambda_i(\theta_s)a(\theta_s,u_s) +\beta} ds,
\end{align*}
This gives   decay estimate \eqref{decay}.
 \end{proof}


\section{Algorithm}\label{sec:algo}

\begin{algorithm}[ht!]
\caption{MMD GAN with Discrete Parametric Kernelized Flows w.r.t. $d_{\alpha,\beta}$}
 \label{alg:MMDFLOWGAN}
\begin{algorithmic}
 \STATE {\bfseries Input:} $\alpha>0$ gradient penalty weight, $\eta$ Learning rate, $n_c$ number of iterations for training the critic, N batch size
 \STATE {\bfseries Initialize} witness function $f(x)=\scalT{v}{\Phi(x)}$, Generator $G_{\theta}$, initialize $v,\theta,\Phi$. Note by $p$ parameters of $v,\Phi$.
 \REPEAT
\STATE \COMMENT{\textcolor{blue}{Update of the witness function of $\mathrm{MMD}^2_{\alpha,\beta}(p,G_{\theta, \#}\nu )$} }\\
 \FOR{$j=1$ {\bfseries to} $n_c$}
 \STATE Sample a minibatch $x_i,i=1\dots N, x_i \sim p$ 
 \STATE Sample a minibatch $z_i,i=1\dots N, z_i \sim \nu$ 
 \STATE Compute $\widehat{\mathrm{MMD}}^2_{\alpha,\beta}(p,G_{\theta,\#}(\nu)) = - M(p,\theta) $
 \STATE $M(p,\theta)= \frac{1}{N}\sum_{i=1}^N f(G(\theta)(z_i)) - \frac{1}{N}\sum_{i=1}^N f(x_i)+ \alpha \nor{\nabla_{\theta}\frac{1}{N}\sum_{i=1}^N f(G_{\theta}(z_i))}^2$ \\
 \COMMENT{We omit the $\beta$ regularizer since SGD achieves similar effet}\\
 \COMMENT{A better estimate of $\nor{\nabla_{\theta} \mathbb{E}_{\nu}f(G_{\theta}(z))}^2$ can be obtained  by using $N^2$ independent samples or by using a variance estimate, but we found this estimate to be enough in practice }
 \IF{Fixed Kernel}
 \STATE $v \gets v -\eta \mathrm{RmsProp}(\nabla_{v}M(p,\theta))$
\ELSIF{Learned Kernel}
\STATE$p \gets p -\eta \mathrm{RmsProp}(\nabla_{p}M(p,\theta))$
 \ENDIF
 \ENDFOR
 \STATE \COMMENT{\textcolor{blue}{Generator update: Parametric  Discrete Kernel Flow w.r.t. $d_{\alpha,\beta}$ on the statistical Manifold}  }
 \STATE Sample ${z_i,i=1\dots N , z_i \sim \nu}$
 \STATE $d_{\theta}\gets  -\nabla_{\theta}\frac{1}{N}\sum_{i=1}^N f(G_{\theta}(z_i))$ 
 \STATE $\theta \gets \theta -\eta \mathrm{RmsProp}(d_{\theta})$
 \UNTIL{$\theta$ converges}
\end{algorithmic}
\end{algorithm}

\section{Additional Experiments and Plots}\label{App:plots}
\subsection{MMD GAN when the witness function is a multilayer  Neural Network: i.e. Learned Kernel }

\begin{figure}[ht!]
\vspace{-1.0em}
    \begin{subfigure}[L]{0.6\textwidth}
    \centering
        \includegraphics[width=\textwidth]{figs/traj_euclidean_kernel.pdf}
        \caption{Fixed Kernel: Trajectories of Flows for $\alpha=0$: Euclidean Gradients Flows}
        \label{fig:0trajbis}
    \end{subfigure}
    \begin{subfigure}[R]{0.4\textwidth}
    \centering
        \includegraphics[width=\textwidth]{figs/mmd_noconv.png}
        \caption{MMD Loss}
        \label{fig:0lossbis}
    \end{subfigure}
    \begin{subfigure}[L]{0.6\textwidth}
    \centering
        \includegraphics[width=\textwidth]{figs/traj_remannian_kernel.pdf}
        \caption{Fixed Kernel: Trajectories of Flows for $\alpha=100$: Kernelized  Gradients Flows w.r.t. $d_{\alpha,\beta}$}
        \label{fig:100trajbis}
    \end{subfigure}
    \begin{subfigure}[R]{0.4\textwidth}
    \centering
        \includegraphics[width=\textwidth]{figs/mmd_remannian_kernel.png}
        \caption{MMD Loss}
        \label{fig:100lossbis}
    \end{subfigure}
       \begin{subfigure}[L]{0.6\textwidth}
    \centering
        \includegraphics[width=\textwidth]{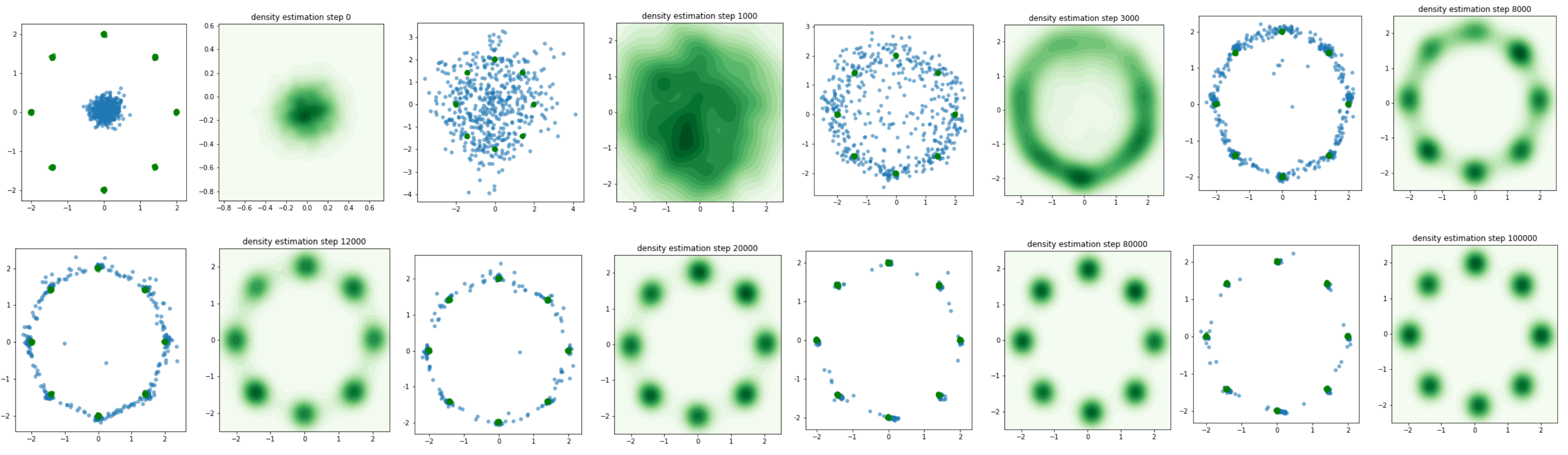}
        \caption{Learned Kernel: Trajectories of Flows for $\alpha=100$: Kernelized  Gradients Flows w.r.t. $d_{\alpha,\beta}$}
        \label{fig:100trajLearned}
    \end{subfigure}
    \begin{subfigure}[R]{0.4\textwidth}
    \centering
        \includegraphics[width=\textwidth]{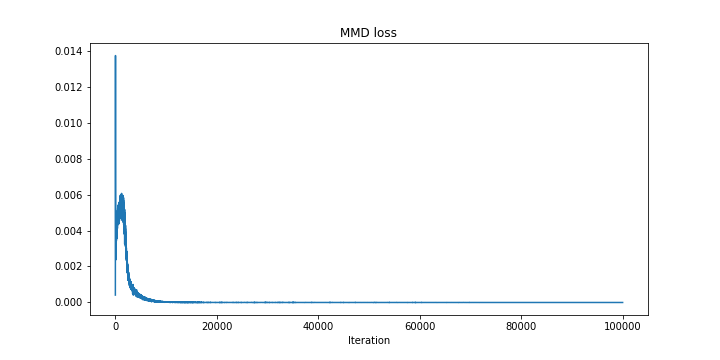}
        \caption{MMD Loss}
        \label{fig:100lossLearned}
    \end{subfigure}
    
    \caption{Trajectories of Kernelized Gradient flows of the $\mathrm{MMD}$ functional for $\alpha=0$ (no gradient regularization) and $\alpha=100$. It is clear that the Riemannian structure induced by  $d_{\alpha,\beta}$, $\alpha>0$ guarantees the convergence, while  we suffer from cycles and mode collapse for $\alpha=0$. When the kernel is learned , i.e. the witness function is a multilayer neural network, the gradient flow exhibits qualitatively similar behavior to the fixed kernel one (compare Fig \ref{fig:100traj}) of the fixed kernel to Fig \ref{fig:100trajLearned} of the learned kernel).
    This suggests possibly a Neural Tangent Kernel Regime that needs further analysis.   }
    \label{fig:convergencebis}
    \vskip -0.12 in    

\end{figure}

\end{document}